%% file: neurips_main.tex
\documentclass{article}

\PassOptionsToPackage{numbers, compress}{natbib}

\usepackage[final]{neurips_2024}

\usepackage[utf8]{inputenc} %
\usepackage[T1]{fontenc}    %
\usepackage{hyperref}       %
\hypersetup{
    colorlinks=true,
    filecolor=magenta,      
    urlcolor=cyan,
    }
\usepackage{url}            %
\usepackage{booktabs}       %
\usepackage{amsfonts}       %
\usepackage{nicefrac}       %
\usepackage{microtype}      %
\usepackage{xcolor}         %

\usepackage{amssymb}
\usepackage{amsmath}
\usepackage{amsthm}
\usepackage{enumitem}
\usepackage{bm}
\usepackage{bbm}
\usepackage{wrapfig}
\usepackage{graphicx}

\newtheorem{theorem}{Theorem}[section]

\newtheorem{lemma}[theorem]{Lemma}

\newcommand{\mS}{\mathcal{S}}
\newcommand{\mA}{\mathcal{A}}
\newcommand{\mT}{\mathcal{T}}
\newcommand{\mO}{\mathcal{O}}
\newcommand{\aht}{\text{aht}}
\newcommand{\naht}{\text{naht}}

\title{N-Agent Ad Hoc Teamwork}

\author{%
  Caroline Wang \\
  Department of Computer Science \\
  The University of Texas at Austin \\
  \texttt{caroline.l.wang@utexas.edu} \\
  \And
  Arrasy Rahman \\
  Department of Computer Science \\
  The University of Texas at Austin \\
  \texttt{arrasy@cs.utexas.edu} \\
  \And
  Ishan Durugkar \\
  Sony AI \\
  \texttt{ishan.durugkar@sony.com} \\
  \And
  Elad Liebman\thanks{Work was done while at SparkCognition.} \\
  Amazon \\
  \texttt{liebelad@amazon.com} \\
  \And
  Peter Stone \\
  Department of Computer Science \\
  The University of Texas at Austin and Sony AI \\
  \texttt{pstone@cs.utexas.edu} \\
}

\begin{document}

\maketitle

\input{contents/0.abstract}
\input{contents/1.intro}
\input{contents/2.background}
\input{contents/3.problem_formulation}

\input{contents/4.motivation}

\input{contents/5.method}
\input{contents/6.experiments}

\input{contents/7.related_work}

\input{contents/8.discussion}
\input{contents/ack}    
\bibliography{refs}
\bibliographystyle{plainnat}

\appendix
\input{contents/appendix}

\end{document}

%% file: contents/0.abstract.tex
\begin{abstract}
    Current approaches to learning cooperative multi-agent behaviors assume relatively restrictive settings.
    In standard fully cooperative multi-agent reinforcement learning, the learning algorithm controls \textit{all} agents in the scenario, while in ad hoc teamwork, the learning algorithm usually assumes control over only a \textit{single} agent in the scenario.
    However, many cooperative settings in the real world are much less restrictive.
    For example, in an autonomous driving scenario,  a company might train its cars with the same learning algorithm, yet once on the road, these cars must cooperate with cars from another company.   
    Towards expanding the class of scenarios that cooperative learning methods may optimally address, we introduce \textit{$N$-agent ad hoc teamwork} (NAHT), where a set of autonomous agents must interact and cooperate with dynamically varying numbers and types of teammates.
    This paper formalizes the problem, and proposes the \textit{Policy Optimization with Agent Modelling} (POAM) algorithm. POAM is a policy gradient, multi-agent reinforcement learning approach to the NAHT problem that enables adaptation to diverse teammate behaviors by learning representations of teammate behaviors. 
    Empirical evaluation on tasks from the multi-agent particle environment and StarCraft II shows that POAM improves cooperative task returns compared to baseline approaches, and enables out-of-distribution generalization to unseen teammates.
\end{abstract}

%% file: contents/1.intro.tex
\section{Introduction}
\label{sec:intro}

Advances in multi-agent reinforcement learning (MARL) \citep{marl-book-albrecht24} have enabled agents to learn solutions to various problems in zero-sum games, social dilemmas, adversarial team games, and cooperative tasks \citep{silver2016mastering, brown-libratus18, koster_commonpool_2024, lin_tizero_2023, rashid18qmix}. 
Within MARL, cooperative multi-agent reinforcement learning (CMARL) is a paradigm for learning agent teams that solve a common task via interaction with each other and the environment \citep{littman1994markov, Panait2005CooperativeML, yuan2023survey}. 
Recent CMARL methods have been able to learn impressive examples of cooperative behavior from scratch in controlled settings, where all agents are controlled by the same learning algorithm \citep{rashid18qmix, son_qtran_2019, baker_emergent_2019}.
A related paradigm for learning cooperative behavior is ad hoc teamwork (AHT). In contrast to CMARL, the objective of AHT is to create a single agent policy that can collaborate with previously unknown teammates to solve a common task \citep{mirsky2022survey, stone2010ad}.

While a large and impressive body of work on CMARL and AHT exists, the current literature has largely examined scenarios in which either complete control over all agents is assumed, or only a single agent is adapted for cooperation \citep{rashid18qmix, son_qtran_2019, lin_tizero_2023, baker_emergent_2019, hu_simplified_2019}. Even learning methods for handling cooperative tasks in open multiagent systems~\citep{yuan2023survey}, which encompass one of the most challenging settings where agents may enter or leave the system anytime, either operate assuming full control over all agents~\citep{jiang2019graph, liu2021coach} or only a single adaptive agent~\citep{rahman_towards_2021, kakarlapudi2022decision, rahman2023general}.

However, real-world collaborative scenarios---e.g. search-and-rescue, or robot fleets for warehouses---might demand agent \textit{subteams} that are able to collaborate with unfamiliar teammates that follow different coordination conventions.
Towards producing agent teams that are more flexible and applicable to realistic cooperative scenarios, this paper formalizes the problem setting of \textbf{$N$-agent ad hoc teamwork} (NAHT), in which a set of autonomous agents  must interact with an uncontrolled set of teammates to perform a cooperative task. 
When there is only a single ad hoc agent, NAHT is equivalent to AHT. 
On the other hand, when all ad hoc agents are jointly trained by the same algorithm and there are no uncontrolled teammates, NAHT is equivalent to CMARL. 
Thus, the proposed problem setting generalizes both CMARL and AHT. 

Drawing from ideas in both CMARL and AHT, we introduce Policy Optimization with Agent Modelling (POAM). POAM is a policy-gradient based approach for learning cooperative multi-agent team behaviors, in the presence of varying numbers and types of teammate behaviors. It consists of (1) an agent modeling network that generates a vector characterizing teammate behaviors, and (2) an independent actor-critic architecture, which conditions on the learned teammate vectors to enable adaptation to a variety of potential teammate behaviors. 
Empirical evaluation on multi-agent particle environment (MPE) and StarCraft II tasks shows that POAM learns to coordinate with a changing number of teammates of various types, with higher competency than CMARL, AHT, and naive NAHT baseline approaches. An evaluation with out-of-distribution teammates also reveals that POAM's agent modeling module improves generalization to out-of-distribution teammates, compared to baselines without agent modeling. 

\begin{figure}[t]
    \centering
    \includegraphics[width=1.0\textwidth]{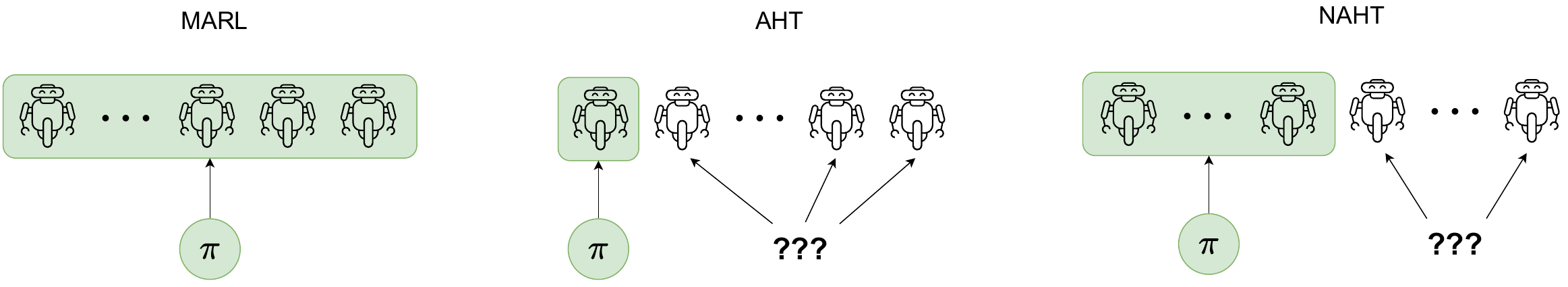}
    \caption{Left: CMARL algorithms  assume full control over all $M$ agents in a cooperative scenario. Center: AHT algorithms assume that only a single agent is controlled by the learning algorithm, while the other $M-1$ agents are uncontrolled and can have a diverse, unknown set of behaviors. Right: NAHT, the paradigm proposed by this paper, assumes that a potentially varying $N$ agents are controlled by the learning algorithm, while the remaining $M-N$ agents are uncontrolled. }
    \label{fig:paradigms}
\end{figure}

%% file: contents/2.background.tex
\section{Background and Notation}
\label{sec:background}
The NAHT problem is formulated within the framework of \textbf{Decentralized Partially Observable Markov Decision Processes}, or Dec-POMDPs  \citep{bernstein_complexity_2002}. 
A Dec-POMDP consists of $M$ agents, a state space $\mS$, action space $\mA$, per-agent observation spaces $O_i$, transition function $\mT: \mS \times \mA \mapsto \Delta(\mS)$ 
\footnote{$\Delta(S)$ denotes the space of probability distributions over set $S$.}, common reward function $r: \mS \times \mA \mapsto \Delta(\mathbb{R})$ (thus  defining a cooperative task), discount factor $\gamma \in [0, 1]$ and horizon $T \in \mathbb{Z}$, which represents the maximum length of an interaction, or episode. 
Each agent observes the environment via its observation function,  $\mO_i: \mS \times \mA \mapsto \Delta(O_i)$.
The state space is factored such that $\mS = \mS_1 \times \cdots \times \mS_M$, where $\mS_i$ for $i \in \{1 \cdots M\}$ corresponds to the state space for agent $i$. 
The action space is defined analogously. 
Denoting $H_{i}$ as its space of localized observation and action histories, agent $i$ acts according to a policy, $\pi_i: H_{i} \mapsto \Delta(\mA_i)$.
The notation $-i$ represents all agents other than agent $i$, and is applied throughout the paper to the mathematical objects introduced above. For example, the notation $\mO_{-i}$ refers to the cross product of the observation space of all agents other than $i$.
In the following, we overload $r$ to refer to both the reward function, and the task defined by that reward function, whereas $r_t$ denotes the reward at time step $t$.

%% file: contents/3.problem_formulation.tex
\section{NAHT Problem Formulation}
\label{sec:problem_formulation}

Drawing from the goals of MARL and AHT \citep{stone2010ad}, the goal of $N$-agent ad hoc teamwork is \textbf{to create a \textit{set} of autonomous agents that are able to efficiently collaborate with both known and unknown teammates to maximize return on a task}.
The goal is formalized below. 

Let $C$ denote a set of ad hoc agents. 
If the policies of the agents in $C$ are generated by an algorithm, we say that the algorithm controls agents in $C$.  Since our intention is to develop algorithms for generating the policies of agents in $C$, we  refer to agents in $C$ as \textit{controlled}.
Let $U$ denote a set of \textit{uncontrolled} agents, which we define as all agents in the environment not included in $C$.\footnote{The original AHT problem statement used the terms ``known'' and ``unknown'' agents. However, it is common for modern AHT learning methods to assume some knowledge about the ``unknown'' teammates during training \citep{mirsky2022survey}. Thus, we instead employ the terms controlled and uncontrolled.}
Following \citeauthor{stone2010ad}, we assume that agents in $U$ are not adversarially minimizing the objective of agents in $C$.

We model an open system of interaction, in which a random selection of $M$ agents from sets $C$ and $U$ must coordinate to perform task $r$. 
For illustration, consider a warehouse staffed by robots developed by companies $A$ and $B$, where there is a box-lifting task that requires three robots to accomplish. If Company A's robots are controlled agents (corresponding to $C$), then some robots from Company A could collaborate with robots from Company B (corresponding to $U$) to accomplish the task, rather than requiring that all three robots come exclusively from  $A$ or $B$. 
Motivated thus, we introduce a \textit{team sampling procedure} $X(U, C)$. At the beginning of each episode, $X$ samples a team of $M$ agents by first sampling $N < M$, then sampling $N$ agents from the set $C$ and $M-N$ agents from $U$. 
We restrict consideration to teams containing at least one controlled agent, i.e $N \geq 1$.
We consider $X$ a problem parameter that is not under the control of any algorithm for generating ad hoc teammates, analogous to the transition function of the underlying Dec-POMDP. 
A more explicit definition of $X$ is provided in  Appendix \ref{app:team_sampling_procedure}.

Without loss of generality, let $C(\theta) = \{\pi_i^{\theta}(.|s)\}_{i=1}^{N}$ denote a \textit{set} of $M$ controlled agent policies parameterized by $\theta$, such that a learning algorithm might optimize for $\theta$. 
Let the $\bm{\pi}^{(M)}$ indicate a \textit{team} of $M$ agents 
and $\bm{\pi}^{(M)} \sim X(U, C)$ indicate sampling such a team from $U$ and $C$ via the team sampling procedure. 
The \textit{objective} of the NAHT problem is to find parameters $\theta$, such that $C(\theta)$ maximizes the expected return in the presence of teammates from $U$:

\begin{equation} \label{eq:obj}
\max_{\theta} \left(
\mathbb{E}_{\bm{\pi}^{(M)} \sim X(U, C(\theta))}\left[\sum_{t=0}^{T} \gamma^t r_t \right]
\right)
.
\end{equation}

Challenges of the NAHT problem include: (1)  coordinating with potentially unknown teammates (\textit{generalization}), and (2) coping with a varying number of uncontrolled teammates (\textit{openness}).

%% file: contents/4.motivation.tex
\section{The Need for Dedicated NAHT Algorithms}
\label{sec:motivation}

Having introduced the NAHT problem, a natural question to consider is whether AHT solutions may optimally address NAHT problems. 
If so, then there would be little need to consider the NAHT problem setting. 
For instance, a simple yet reasonable approach consists of directly using an AHT policy to control as many agents as required in an NAHT scenario. 
This section illustrate the limitations of the aforementioned approach by giving a concrete example of a matrix game where (1) an AHT policy that is learned in the AHT ($N=1$) scenario is unlikely to do well in an NAHT  scenario where $N=2$, and (2) even an \textit{optimal} AHT policy is suboptimal in the $N=2$ setting.

Define the following simple game for $M$ agents: at each turn, each agent $a_i$ picks one bit $b_i \in \{0, 1\}$; at the end of each turn, all the bits are summed $s = \sum_i b_i$. 
The team wins if the sum of the chosen bits is exactly $1$. We denote the probability of winning by $P(s=1)$.
Suppose the uncontrolled agents follow a policy that independently selects $1$ with probability $p = \frac{1}{M}$.\footnote{Lemma \ref{lemma:optimal-p-static} shows that $p=\frac{1}{M}$ is the optimal value of $p$ for a team composed of such agents.}
In the following, we consider the three agent case, $M=3$, for simplicity.

In the AHT problem setting, a learning algorithm assumes control of only a single agent. Let $p_{\aht}$  denote the probability with which the AHT agent selects 1. Given the aforementioned team of uncontrolled agents, we show that \textit{any} value of $p_{\aht}$ results in the same probability of winning, which occurs because the probability of winning, $P(s=1) = \frac49$, is independent of $p_{\aht}$ (Lemma \ref{lemma:aht-scenario}).

Next, consider an NAHT scenario where a learning algorithm must define the actions of two out of three agents. 
Suppose that the same AHT policy is used to control both agents: both agents select 1 with probability $p_{\aht}$. 
Above, we demonstrated that an AHT algorithm trained in the $N=1$ scenario could result in learning any $p_{\aht}$. 
However, in the $N=2$ setting, we show that the optimal AHT policy $p_{\aht} = \frac13$ and the winning probability for this policy is $P(s=1) = \frac49$ (Lemma \ref{lemma:aht-in-naht-setting}).

Finally, we show there exists an NAHT policy that controls both agents and obtains a higher winning probability.
Consider the policy where one controlled agent always plays 0, while the other plays 1 with probability $p_{\naht}$. 
Lemma \ref{lemma:two-player-setting} shows that the optimal $p_{\naht} = 1$, and the probability of winning $P(s=1) = \frac23 > \frac49$.
Thus, we have exhibited an NAHT scenario where an AHT policy that is optimal when $N=1$, performs worse than a simple NAHT joint policy in the $N=2$ setting. Empirical validation of the prior results are provided in Appendix \ref{app:supp_results:need_for_naht_exp}.

%% file: contents/5.method.tex
\section{Policy Optimization with Agent Modeling (POAM)}
\label{sec:method}

\begin{wrapfigure}{R}{0.55\textwidth} %
    \vspace{-10pt}
    \includegraphics[width=0.55\textwidth]{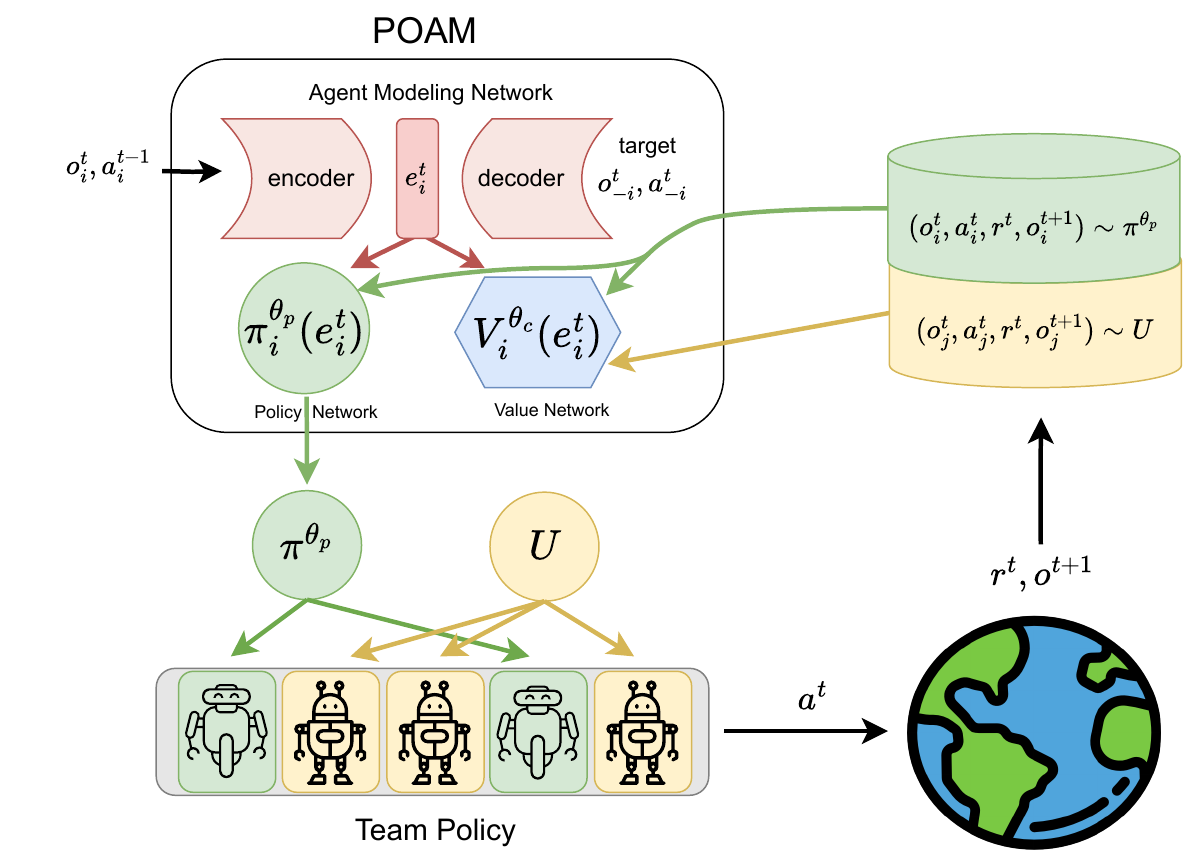}
    \caption{POAM trains a single policy network $\pi^{\theta_p}$, which characterizes the behavior of all controlled agents (green), while uncontrolled agents (yellow) are drawn from $U$. 
    Data from both controlled and uncontrolled agents is used to train the value network, $V^{\theta_c}_i$
    while the policy is trained on data from the controlled agents only. 
    The policy and value function are both conditioned on a learned team embedding vector, $e^t_i$. 
    }
    \vspace{-10pt}
    \label{fig:method}
\end{wrapfigure}

This section describes the proposed Policy Optimization with Agent Modeling (POAM) method, which trains a collection of NAHT agents that can adaptively deal with different collections of unknown teammates. POAM relies on an \textit{agent modeling network} to initially build an embedding vector characterizing teammates encountered during an interaction. Adaptive agent policies that can maximize the controlled agents' returns are then learned by training a \textit{policy} conditioned on the environment observation and team embedding vector. 
To enable controlling a varying number of agents while learning in a sample-efficient manner, POAM adopts the independent learning framework with full parameter sharing.
The training processes for agent modeling and policy networks are described in Sections~\ref{sec:AMNetwork} and~\ref{sec:PolNetwork} respectively, while an illustration of how POAM trains NAHT agents is provided in Figure~\ref{fig:method}. 

\subsection{Agent Modeling Network}
\label{sec:AMNetwork}

Designing adaptive policies that enable NAHT agents to achieve optimal returns against any team of uncontrolled agents drawn from some set $U$, requires information on the encountered team's unknown behavior. However, in the absence of prior knowledge about uncontrolled teammates' policies, local observations from a single timestep may not contain sufficient information regarding the encountered team. To circumvent this lack of information, POAM's agent modeling network plays a crucial role in providing \textit{team embedding vectors} that characterize the observed behavior of teammates in the encountered team.

We identify two main criteria for desirable team embedding vectors. First, team embedding vectors should identify information regarding the unknown state and behavior of other agents in the environment (both controlled and uncontrolled). Second, team embedding vectors should ideally be computable from the sequence of local observations and actions of the team. Fulfilling both requirements provides an agent with useful information for decision-making in NAHT problems, even under partial observability.

For each controlled agent, POAM produces informative team embedding vectors by training a model with an encoder-decoder architecture, illustrated by red components in Figure~\ref{fig:method}. 
For ease of presentation, the encoder-decoder models for controlled agent $i$ will be referred to without the index $i$.
The encoder, 
$f^{\text{enc}}_{\theta^{e}}: H_{i} \mapsto \mathbb{R}^{n}$, is parameterized by $\theta^{e}$ and processes the modeling agent's history of local observations and actions up to timestep $t$, $h_i^{t} = \{o_{i}^{k}, a_{i}^{k-1}\}_{k=1}^{t}$, to compute a team embedding vector of dimension $n$, $e^{t}_{i} \in \mathbb{R}^n$ that characterizes the modeled agents. This reliance on local observations helps ensure that the agent modeling network can operate without having access to the environment state that is unavailable under partial observability. 
The team embedding vector is decoded by two decoder networks: the observation decoder, $f^{\text{dec}}_{\theta^{o}}: \mathbb{R}^{n} \mapsto O_{-i}$ and the action decoder, $f^{\text{dec}}_{\theta^{a}}: \mathbb{R}^{n} \mapsto \Delta(A_{-i})$. The decoder networks are respectively trained to predict the observations and actions of all other agents on the team at timestep $t$, $(o_{-i}^{t},a_{-i}^{t})$, to encourage $e^{t}_{i}$ to contain relevant information for the current NAHT agent's decision-making process. 
While the observation decoder directly predicts the observed $-i$ observations, the action decoder predicts the parameters of a probability distribution over the $-i$ agents' actions, $p(a^t_{-i}; f^{\text{dec}}_{\theta^{a}}(f^{\text{enc}}_{\theta^{e}}(h_i^{t})))$, where an appropriate distribution for $p$ should be selected by the system designer.

Concretely, agent $i$'s encoder-decoder model is trained to minimize a maximum likelihood loss over all teammates' observations and actions, given its own local observations and actions. As the experimental setting in this paper considers continuous observations and discrete actions, the observation loss is a mean squared error loss, while the action loss is the negative log likelihood of the $-i$ agents' actions, under the Categorical distribution.

\begin{equation}
\label{Eq:AELossFunction}
    L_{\theta^{e},\theta^{o},\theta^{a}}(h_i^{t}, o_{-i}^t, a_{-i}^t) = \Big(||f^{\text{dec}}_{\theta^{o}}(f^{\text{enc}}_{\theta^{e}}(h_i^{t})) - o_{-i}^{t}||^{2} - \text{log}(p(a^{t}_{-i};f^{\text{dec}}_{\theta^{a}}(f^{\text{enc}}_{\theta^{e}}(h_i^{t}))))\Big).
\end{equation}

\subsection{Policy and Value Networks}
\label{sec:PolNetwork}

POAM relies on an actor-critic approach to train agent policies, where the policy and critic are both conditioned on the teammate embedding  described in Section~\ref{sec:AMNetwork}.

The policy network of agent $i$, $\pi_{i}^{\theta^{p}}: H_{i}\times\mathbb{R}^{n} \mapsto \Delta(A_{i})$, is parameterized by $\theta^p$, and uses the NAHT agent's local observation, $o^{t}_{i}$, and the team embedding from the encoder network, $e^{t}_{i}$, to compute a policy followed by the NAHT agents. 
Conditioning the policy network on $e^{t}_{i}$ allows an NAHT agent to change its behaviour based on the inferred characteristics of encountered agents. 
When training the policy network, we also rely on a value (or critic) network, $V_i^{\theta^{c}}: H_{i}\times\mathbb{R}^{n} \mapsto \mathbb{R}$, parameterized by $\theta^c$, which measures the expected returns given $h_i^{t}$, and $e^{t}_{i}$. The value network serves as a baseline to reduce the variance of the gradient updates, while conditioning on the learned teammate embeddings for similar reasons to the policy. 

POAM then trains the policy and value networks using an approach based on the Independent PPO algorithm~\citep{yu2021surprising} (IPPO). IPPO is selected as the base MARL algorithm for two reasons. First, using an independent MARL method circumvents the need to deal with the changing number of agents resulting from environment openness. Second, IPPO has been demonstrated to be effective on various MARL tasks. 
To improve learning efficiency and enable information sharing between agents, full parameter sharing is employed for all neural networks. 
POAM trains the value network to produce accurate state value estimates by minimizing the following loss function:
\begin{equation}
\label{Eq:ValueLossFunction}
    L_{\theta^{c}}(h_i^{t}) = \dfrac{1}{2}\Big(V_i^{\theta^{c}}(h_i^{t}, f^{\text{enc}}_{\theta^{e}}(h_i^{t}))-\hat{V_i}^{t}\Big)^2,
\end{equation}
where $\hat{V_i}^{t}$ is the TD($\lambda$) return.
The policy network is analogously trained to minimize the PPO loss function \citep{Schulman2017ProximalPO}, but where the policy additionally conditions on the team embeddings.

\paragraph{Leveraging data from uncontrolled agents}
In the NAHT setting, we assume access to the \textit{joint} observations and actions generated by the current team deployed in the environment \textit{at training time only}, where the team consists of a mix of controlled and uncontrolled agents. 
This assumption provides an opportunity to learn useful cooperative behaviors more quickly, by bootstrapping based on transitions from the initially more competent, uncontrolled teammate policies, as also observed by  \citet{rahman_towards_2021}.

POAM leverages data from both controlled and uncontrolled agents to train the value network---in effect, treating the uncontrolled agents as exploration policies. 
Note that this aspect of POAM is a significant departure from PPO, which is a fully on-policy algorithm.
Since the policy update is highly sensitive to off-policy data, only data from the controlled agents is used to train the policy network.

%% file: contents/6.experiments.tex
\section{Experiments and Results}
\label{sec:exp}

This section presents an empirical evaluation of POAM and baseline approaches across different NAHT problems. 
We investigate three questions and foreshadow the conclusions: 
\begin{itemize}[topsep=0pt,itemsep=0pt,
leftmargin=2pt,labelindent=0pt]
    \item[] Q1: Does POAM learn to cope with uncontrolled teammates with higher sample efficiency and asymptotic return than baselines? (Usually)
    \item[] Q2: Does POAM improve generalization to previously unseen and out-of-distribution teammates, compared to baselines? (Yes)
    \item[] Q3: Can we verify that the two key ideas of POAM---agent modelling and use of data from uncontrolled agents---work as desired and contribute positively towards POAM's performance? (Yes) 
\end{itemize}

Full implementation details, including hyperparameter values and additional empirical results, appear in the Appendix. Our code is available at \url{https://github.com/carolinewang01/naht}.

\subsection{Experimental Design}
\label{sec:ExpDesign}
In the following, we summarize the experimental design, including the particular NAHT problem instance, training procedure, experimental domain, and baselines. Details on evaluation metrics are provided in Appendix \ref{app:eval_details}.
\paragraph{A Practical Instantiation of NAHT}
Similarly to AHT, the NAHT problem can be parameterized by the amount of knowledge that controlled agents have about uncontrolled agents, and whether uncontrolled agents can adapt to the behavior of controlled agents \citep{barrett2012analysis}. 
Furthermore, as a direct result of the fact that an NAHT algorithm may control more than a single agent, the NAHT problem may also be parameterized by whether the controlled agents are homogeneous, whether they can communicate, and what the team sampling procedure is. 
While the fully general problem setting allows for heterogeneous, communicating, controlled agents that have no knowledge of the uncontrolled agents, as a first step, this paper focuses on a special case of the NAHT problem, where agents are homogeneous, non-communicating, may learn about uncontrolled agents via interaction, and where the team sampling procedure consists of a uniform random sampling scheme from $U$ and $C$ (see Appendix \ref{app:exp_team_sampling_proc} for details). This case is designed primarily to assess whether controlled \textit{subteams} may outperform independent controlled agents, when cooperating with multiple types of uncontrolled agents. 
We leave consideration of broader NAHT scenarios for future work. 

\paragraph{Generating Uncontrolled Teammates}
To generate a set of uncontrolled teammates $U$, the following MARL algorithms are used to train agent teams: VDN~\citep{sunehag18vdn}, QMIX~\citep{rashid18qmix}, IQL~\citep{Tan1997MultiAgentRL}, IPPO, and MAPPO~\citep{yu2021surprising}.
We verify that the generated team behaviors are diverse by checking that (1) teams trained by the same algorithm learn non-compatible coordination conventions, and (2) teams trained by different algorithms also learn non-compatible coordination conventions (Appendix \ref{app:supp:coord_conventions}). The emergence of diverse teammate behaviors from training agents using different MARL algorithms under different seeds aligns with the findings from  \citet{strouse_fcp_2022}.

Let $U_{train}$ denote the set of uncontrolled teammates used to train all (N)AHT methods. $U_{train}$ consists of five teams, where each individual team is trained via VDN, QMIX, IQL, IPPO and MAPPO, respectively. $U_{test}$ consists of a set of holdout teams trained via the same MARL algorithms, but that have not been seen during training. 
The experimental results reported in Sections \ref{sec:exp_core_results} and \ref{sec:poam_ablation} are computed with respect to $U_{train}$ only, while the experimental results in Section \ref{sec:ood_generalization} use $U_{test}$.

\paragraph{Experimental Domains}
Experiments are conducted on a predator-prey \texttt{mpe-pp} task implemented within the multi-agent particle environment ~\citep{lowe2017multi}, and the \texttt{5v6, 8v9, 10v11, 3s5z} tasks from the StarCraft Multi-Agent Challenge (SMAC) benchmark~\citep{samvelyan19smac}. 
On the \texttt{mpe-pp} task, three predators must cooperatively pursue a pretrained prey agent. The team receives a reward of +1 per time step that two or more predators collide with the prey. 
On the SMAC tasks, a team of allied agents must defeat a team of enemy agents controlled by the game server. For each task, the first number in the task name indicates the number of allied agents, while the second indicates the number of enemy agents. The team is rewarded for defeating enemies, with a large bonus for defeating all enemies.
See Appendix \ref{app:exp_setting} for full details.

\paragraph{Baselines}

As NAHT is a new problem proposed by this paper, there are no prior algorithms that are directly designed for the NAHT problem. 
Therefore, we construct three baselines to contextualize the performance of POAM.
All methods employ full parameter sharing~\citep{papoudakis_benchmarking_2021}.

\begin{itemize}[leftmargin=*]
    \item \textit{Naive MARL}: various well-known MARL algorithms are considered, including both independent and centralized training with decentralized execution algorithms~\citep{hernandez-leal_dmarl-survey_2019}. The algorithms evaluated here include IQL~\citep{claus_dynamics_1998}, VDN~\citep{sunehag18vdn}, QMIX~\cite{rashid18qmix}, IPPO, and MAPPO~\citep{yu2021surprising}. The MARL baselines are trained in self-play and then evaluated in the NAHT setting. In the following, only the performance of the \textit{best} naive MARL baseline is reported.
    \item \textit{Independent PPO in the NAHT setting} (IPPO-NAHT): IPPO is a policy gradient MARL algorithm that directly generalizes PPO~\citep{Schulman2017ProximalPO} to the multi-agent setting. It was found to be surprisingly effective on a variety of MARL benchmarks~\cite{yu2021surprising}. In contrast to the naive MARL baselines, IPPO-NAHT is trained using the NAHT training scheme presented in Section \ref{sec:ExpDesign}. The variant considered here employs full parameter sharing, where the actor is trained on data only from controlled agents, but the critic is trained using data from both controlled and uncontrolled agents. The latter detail is a key algorithmic feature which POAM also employs, but is an extension from the most naive version of PPO (see Section \ref{sec:poam_ablation}). IPPO can be considered an ablation of POAM, where the agent modeling module is removed.
    \item{\textit{POAM in the AHT setting} (POAM-AHT)}: As considered in Section \ref{sec:motivation}, a natural baseline approach to the NAHT problem is to use AHT algorithms that train only a single controlled agent, and copy these policies as many times as needed in the NAHT setting. To evaluate the intuition that AHT policies do not suffice for the NAHT problem setting, we consider an AHT version of POAM that is trained identically to POAM, but where the number of controlled agents is always one ($N = 1$) during training. 
    Note that POAM-AHT is equivalent to the AHT algorithm introduced by~\citet{papoudakis2020liam}, LIAM.
\end{itemize}

\subsection{Main Results}
\label{sec:exp_core_results}

\begin{figure}[t]
    \centering
    \begin{minipage}{\textwidth}
        \centering
        \includegraphics[width=\textwidth]{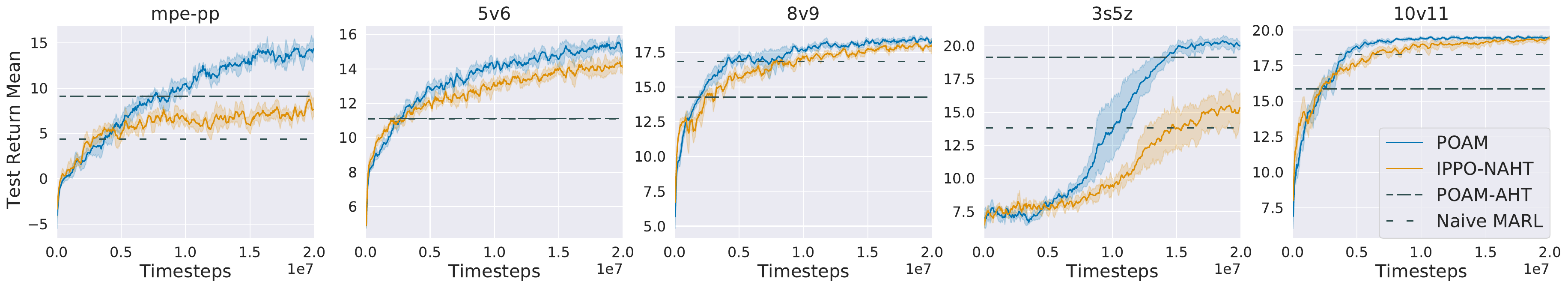}
    \end{minipage}
    \caption{POAM consistently improves over the baselines of IPPO-NAHT, POAM-AHT, and the best naive MARL baseline in all tasks, in either sample efficiency or asymptotic return. 
    }
    \label{fig:core_results}
    \vspace{-10pt}
\end{figure}

This section addresses Q1---that is, whether POAM learns to cope with uncontrolled teammates with greater sample efficiency or asymptotic returns, compared to baselines. 
Figure \ref{fig:core_results} shows the learning curves of POAM and IPPO-NAHT, and the test returns achieved by the best naive MARL baseline and POAM-AHT, on all tasks.\footnote{See the Appendix for tables providing the performances of all naive MARL methods, across all tasks.} 
All learning curves consist of the mean test returns across \textbf{five} trials, while the shaded regions reflect the \textbf{95\% confidence intervals}. 
 
We find that for all tasks, POAM outperforms baselines in  terms of asymptotic return for three out of five tasks (\texttt{mpe-pp, 5v6, 3s5z}).
For all tasks, POAM's initial sample efficiency is similar to that of IPPO-NAHT for the first few million steps of training, after which POAM displays higher return. 
We attribute the initial similarity in sample efficiency to the initial cost incurred by learning team embedding vectors, which once learned, improves learning efficiency.
IPPO-NAHT, which can be viewed as an ablation of POAM with no agent modeling, is the next best performing method. 
Although IPPO-NAHT has generally poorer sample efficiency than POAM, the method converges to approximately the same returns on two out of five tasks (\texttt{8v9} and  \texttt{10v11}).  
While the agent modeling module of POAM provides team embedding vectors to the policy learning process, the embeddings are themselves produced from each agent's own observation. Since no additional information is provided to POAM agents, it is unsurprising that IPPO-NAHT can converge to similar solutions as POAM, given enough training steps.

\begin{wrapfigure}{r}{0.35\textwidth} %
    \centering
    \vspace{-20pt}
    \includegraphics[width=0.35\textwidth]{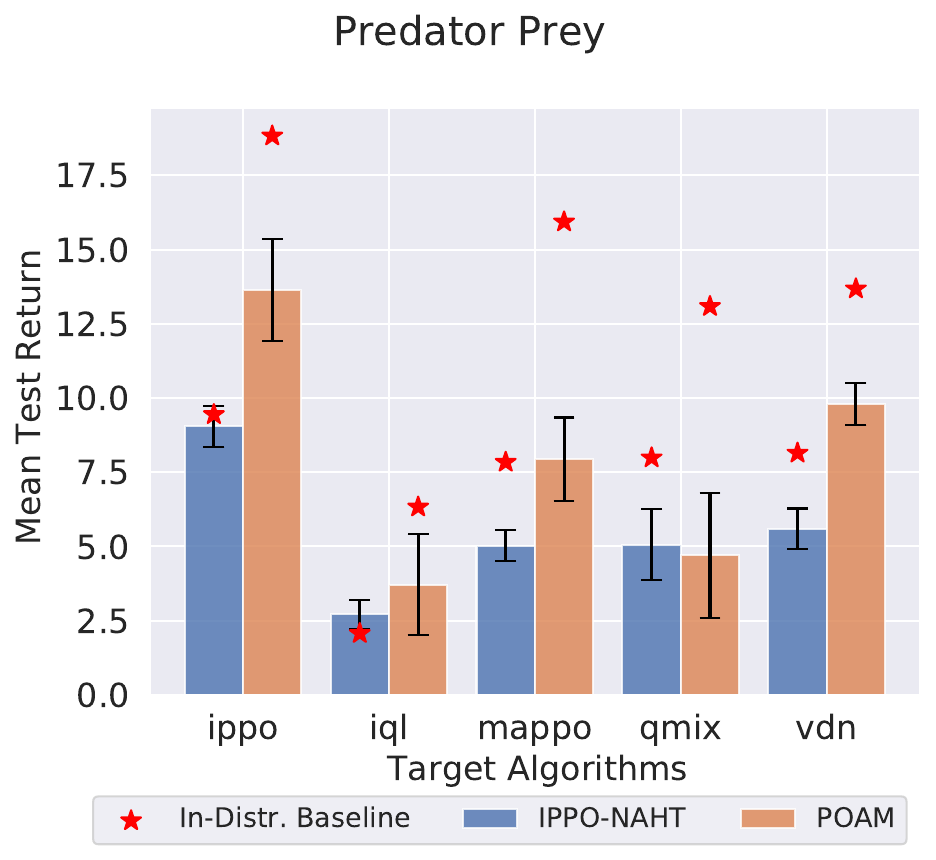}
    \caption{Test returns achieved by POAM and IPPO-NAHT, when paired with out-of-distribution teammates. POAM has improved generalization to OOD teammates, compared with IPPO-NAHT.}
    \label{fig:mpe-pp_ood_gen}
    \vspace{-30pt}
\end{wrapfigure}

Finally, while the best naive MARL baseline and POAM-AHT learn good solutions on some tasks, neither method consistently performs well across all tasks.
Overall, POAM discovers the most consistently performant policies compared to baseline methods, in a relatively sample-efficient manner. We conclude that (1) agent modeling improves learning efficiency, and (2) direct duplication of AHT-trained policies is less effective than methods that co-train agents for NAHT.

\subsection{Out of Distribution Generalization}
\label{sec:ood_generalization}

The AHT literature commonly assumes that uncontrolled teammates of interest may be interacted directly with during training \citep{papoudakis2020liam, barrett2017plastic, rahman_towards_2021}; 
experiments in the prior section were conducted under this assumption. 
However, in realistic scenarios, it may be challenging to enumerate all teammates likely to be encountered in the wild. 
This section examines the effectiveness of POAM under a true NAHT \emph{evaluation} scenario, where POAM agents must coordinate with teammates that were not available at training time and are out-of-distribution (OOD) (Q2).
Here, OOD teammates are created by running MARL algorithms with different random seeds than those used to generate train-time teammates. 

Figures~\ref{fig:mpe-pp_ood_gen} and \ref{fig:sc2_ood_gen} show the mean and 95\% confidence intervals of the test return achieved by POAM, compared to IPPO-NAHT, when the algorithm in question is paired with previously unseen seeds of IPPO, IQL, MAPPO, QMIX, and VDN. 
For each type of teammate, the performance of IPPO-NAHT/POAM against the exact teammates seen during training is shown as the in-distribution baseline.
Both POAM and IPPO-NAHT consistently exhibit reduced performance against the OOD teammates, compared to their respective in-distribution performances.
In three out of five tasks (\texttt{mpe, 5v6, 3s5z}), POAM has a significantly higher return than IPPO-NAHT, while the remaining two tasks exhibit a smaller improvement. 
In Appendix \ref{app:supp:alt_ood_gen}, similar findings are presented with an alternative OOD teammate generation strategy, where the set of five MARL algorithms used to generate uncontrolled teammates (IPPO, IQL, MAPPO, QMIX, VDN) are divided into train/test sets.

\subsection{A Closer Look at POAM}
\label{sec:poam_ablation}
Two key aspects of POAM are the teammate modeling module, and the use of data from uncontrolled agents to train the critic (Q3). We study the impact of both aspects on POAM's sample efficiency, focusing on the \texttt{mpe-pp} and \texttt{5v6} tasks. Results on \text{5v6} may be found in Appendix \ref{app:supp_results}.

\paragraph{Teammate modeling performance}
In the NAHT training/evaluation setting, a new set of teammates is sampled at the beginning of each episode.
Therefore, an important subtask for a competent NAHT agent is to rapidly model the distribution and type of teammates at the beginning of the episode, to enable the policy to exploit that knowledge as the episode progresses.
This task is especially challenging in the presence of partial observability (a property of the SMAC tasks). 
To address the above challenges, POAM employs a recurrent encoder, which encodes the POAM agent's history of observations and actions to an embedding vector, and a (non-recurrent) decoder network, which predicts the egocentric observations and action distribution for all teammates. 

\newpage

\begin{wrapfigure}{l}{0.53\textwidth} %
    \centering
    \includegraphics[width=0.53\textwidth]{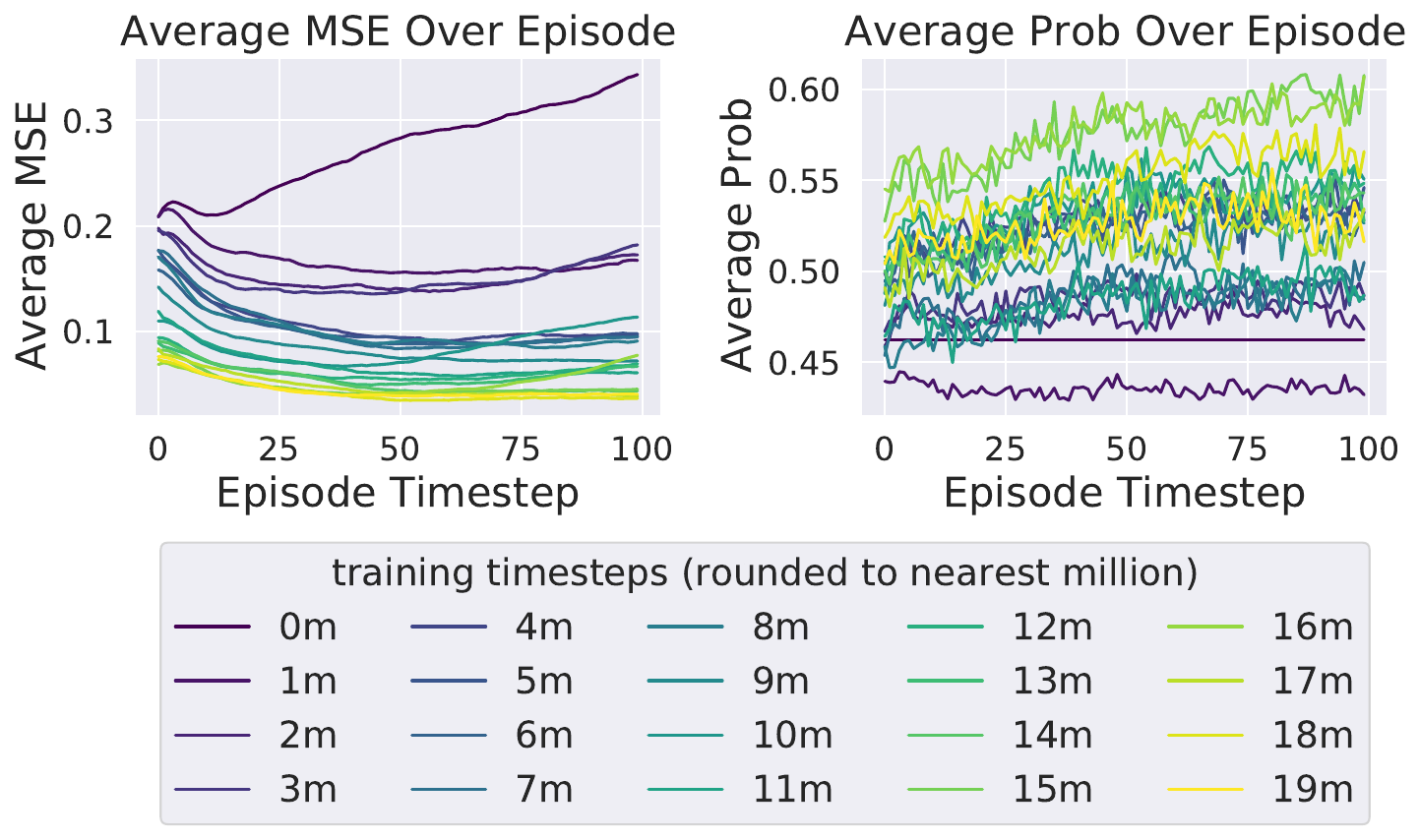}
    \caption{Evolution of a POAM agent's within-episode mean squared error (left) and within-episode probability of actions of modeled teammates (right), over the course of training on  \texttt{mpe-pp}.}
    \label{fig:mpe-pp_poam_within_episode}
\end{wrapfigure}

A natural question then, is whether the learned teammate embedding vectors actually improve over the course of an episode. 
Figures~\ref{fig:mpe-pp_poam_within_episode} and \ref{fig:5v6_poam_within_episode} depict the within-episode \textit{mean squared error} (MSE) of the observation predicted by the decoder, and the within-episode \textit{probability} of the action actually taken by the modeled teammates, according to the decoder's modeled action distribution. 
Each curve corresponds to a different training checkpoint of a single run.

For both tasks, we observe that the average MSE decreases over the course of training, while the average probability of the taken actions increases. For later training checkpoints, we observe that the average MSE actually decreases \textit{within} an episode, while the average probability increases, reflecting the increased confidence of the agent modelling module as more data is observed about teammates. 
Thus, we conclude that POAM is able to cope with the challenges introduced by the sampled teammates and partial observability, to learn accurate teammate models.

\paragraph{Impact of data from non-controlled agents}

Recall that both POAM and IPPO-NAHT update the value network using data from both the controlled and uncontrolled agents.
As Figures~\ref{fig:mpe-pp_critic_masking} and \ref{fig:5v6_critic_masking} show, this algorithmic feature results in a significant performance gain over training using on-policy data only, for both POAM and IPPO-NAHT.

\begin{wrapfigure}{r}{0.35\textwidth} %
    \vspace{-20pt}
    \centering
    \includegraphics[width=0.35\textwidth]{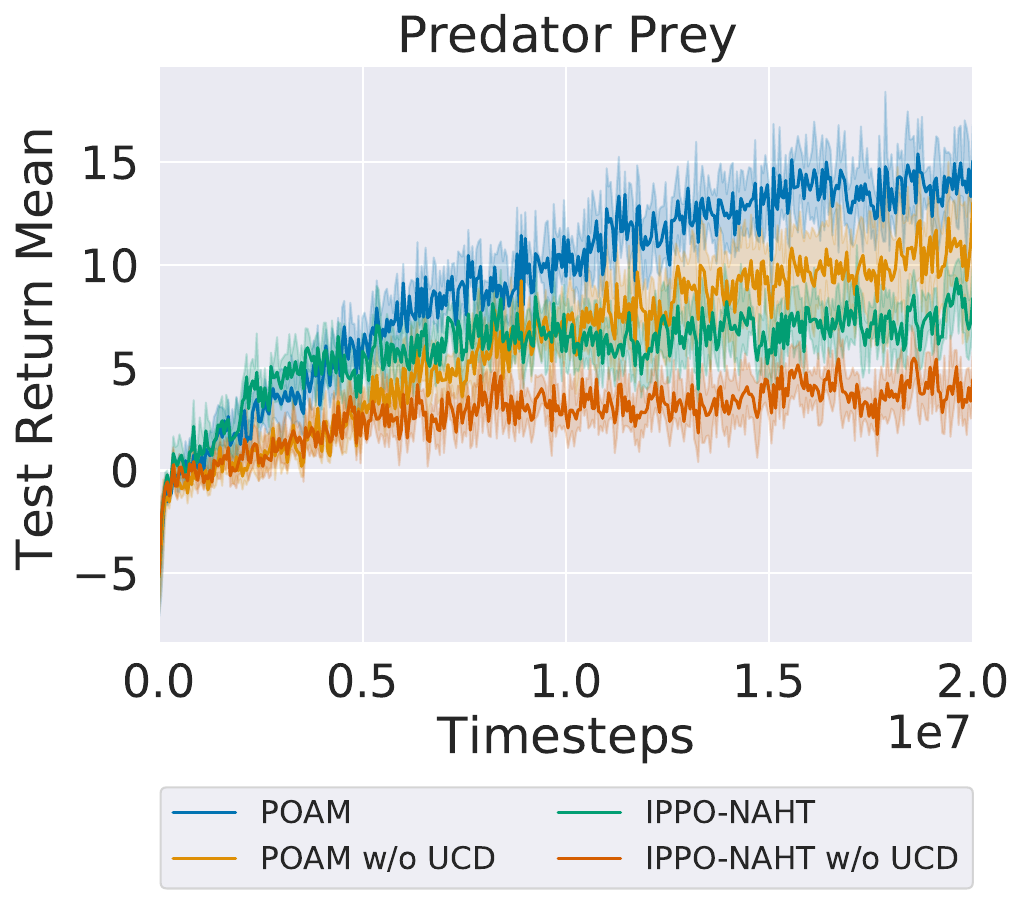}
    \caption{Learning curves of POAM and IPPO-NAHT, where the value network is trained w/w.o. \underline{uncontrolled agents' data} (UCD). }
    \label{fig:mpe-pp_critic_masking}
    \vspace{-20pt}
\end{wrapfigure}

%% file: contents/7.related_work.tex
\section{Related Works}
\label{sec:related_works}
This section summarizes literature in areas most closely related to NAHT, namely, ad hoc teamwork, zero-shot coordination, evaluation of cooperative capabilities, agent modeling, and CMARL.
\paragraph{Ad Hoc Teamwork \& Zero-Shot Coordination.} Prior works in ad hoc teamwork~\citep{stone2010ad} and zero-shot coordination (ZSC)~\citep{hu2020other} explored methods to design adaptive agents that can optimally collaborate with unknown teammates. While they both highly resemble the NAHT problem, existing methods for AHT~\citep{rahman_towards_2021, mirsky2022survey} and ZSC~\citep{hu2020other, lupu2021trajectory} have been limited to single-agent control scenarios. We argue that direct, naive applications of AHT and ZSC techniques to our problem of interest are ineffective---see the discussion in Section \ref{sec:motivation} and results in Section \ref{sec:exp}. 

Recent research in AHT and ZSC utilizes neural networks to improve agent collaboration within various team configurations. These recent works mostly focus on two approaches. The first approach trains the agent to adapt to unknown teammates by characterizing teammates' behavior as fixed-length vectors using neural networks and learning a policy network conditioned on these vectors~\citep{rahman_towards_2021, papoudakis2020liam, zintgraf2021deep}. The second designs teammate policies that maximize the agent's performance when collaborating with diverse teammates~\citep{lupu2021trajectory, rahman_towards_2021}. Our work builds on the first category, extending it to control multiple agents amid the existence of unknown teammates. While this also offers a potential path for robust NAHT agents, designing teammate policies for training will be kept as future work.

\vspace{-5pt}
\paragraph{Evaluating Agents' Cooperative Capabilities.} Beyond training agents to collaborate with teammates having unknown policies, researchers have developed environments and metrics to assess cooperative abilities. The Melting Pot suite~\citep{leibo2021scalable,agapiou2022melting} mostly evaluates controlled agents' ability to maximize utilitarian welfare against unknown agents. However, this evaluation suite focuses on mixed-motive games where agents may have conflicting goals. This contrasts with our work's scope of fully cooperative settings, where all agents share the same reward function. MacAlpine et al. ~\citep{LNAI17-MacAlpine} also explored alternative metrics to measure agents' cooperative capabilities while disentangling the effects of their overall skills in drop-in RoboSoccer.

\vspace{-5pt}
\paragraph{Agent Modeling.} 
Agent modeling enables agents to characterize other agents based on their actions~\citep{albrecht2018autonomous}. Such characterizations could attempt to directly infer modeled agents' actions,  goals~\citep{hanna2021interpretable}, or policies~\citep{papoudakis2020liam}. The modeled attributes have been used in cooperative, competitive, and general sum settings to inform update rules~\citep{foerster2018lola, Shen2019RobustOM} or to directly inform decision making~\citep{papoudakis2020liam}. POAM relies on agent modeling to provide important teammate information for decision-making when collaborating with unknown teammates.

\vspace{-5pt}
\paragraph{Cooperative MARL (CMARL).} 
CMARL explores algorithms for training agent teams on fully cooperative tasks. Some existing methods focus on credit assignment and decentralized control~\citep{foerster2018counterfactual, rashid18qmix}. Other works in CMARL also leverage parameter sharing and role assignment (e.g.,~\cite{christianos2021scaling, yang2022ldsa}) to decide an optimal division of labor between agents.  However, these techniques assume control over all existing agents during training and evaluation, which limits their effectiveness in settings with unseen or uncontrolled teammates, as shown in prior work~\citep{vezhnevets2020options, hu2020other, rahman_towards_2021}.

%% file: contents/8.discussion.tex
\section{Conclusion}
\label{sec:discussion}

 This paper proposes and formulates the problem of $N$-agent ad hoc teamwork (NAHT), a generalization of both AHT and MARL. 
 It further proposes a multi-agent reinforcement learning algorithm to train NAHT agents called POAM, and develops a procedure to train and evaluate NAHT agents. 
 POAM is a policy gradient approach that uses an encoder-decoder architecture to perform teammate modeling, and leverages data from uncontrolled agents for policy optimization. 
 Empirical validation on MPE and StarCraft II tasks shows that POAM consistently improves over baseline methods that naively apply existing MARL and AHT approaches, in terms of sample efficiency, asymptotic return, and generalization to out-of-distribution teammates. 
 
 \paragraph{Limitations and Future Work. }
This paper addresses a special case of the NAHT problem, with homogeneous and non-communicating agents. 
POAM, which employs full parameter sharing, may not perform well in settings with heterogeneous agents, or in settings that require highly differentiated roles. 
POAM also does not leverage centralized state information or allow communication between controlled agents. 
Incorporating this information might enable learning improved NAHT policies. 
Further, POAM's actor update is purely on-policy, and therefore cannot leverage data generated by uncontrolled agents. Future work might consider employing off-policy methods to exploit the uncontrolled agent data. 
In addition to the directions suggested by POAM's limitations, algorithmic ideas from AHT, such as diversity-based teammate generation \citep{lupu2021trajectory} and teammate-model-based planning methods \citep{barrett2017plastic}, also suggest rich avenues for future work. 
Having introduced the NAHT problem in this work, we hope the community explores the many potential directions to design even better NAHT algorithms by considering advances in MARL, AHT, and agent modeling.

%% file: contents/ack.tex
\begin{ack}
This work has taken place in the Learning Agents Research
Group (LARG) at UT Austin.  LARG research is supported in part by NSF
(FAIN-2019844, NRT-2125858), ONR (N00014-18-2243), ARO
(W911NF-23-2-0004, W911NF-17-2-0181), Lockheed Martin, and UT Austin's
Good Systems grand challenge.  Peter Stone serves as the Executive
Director of Sony AI America and receives financial compensation for
this work.  The terms of this arrangement have been reviewed and
approved by the University of Texas at Austin in accordance with its
policy on objectivity in research.
\end{ack}

%% file: contents/appendix.tex
\onecolumn
\section{Appendix}
\label{sec:appendix}

\subsection{An Example of the Team Sampling Procedure}
\label{app:team_sampling_procedure}
Let $g(k, \mu, s)$ represent a function for randomly selecting $k$ elements from a generic set $\mu$, conditioned on a state $s \in \mS$. 
For instance, $g$ might be the distribution, \textit{Multinomial}$(k, |\mu|, \bm{\phi}(s))$, for selecting $k$ elements with replacement from the set $\mu$, parameterized by the state-dependent probability logits $\bm{\phi}(s) \in [0, 1]^{|\mu|}$. 
Allowing $g$ to depend on an initial state enriches the representable open interactions. Returning to our robot warehouse example, a state-dependent $g$ might enable us to model that  robots suitable for heavy loads are more likely to be present near the loading dock area. 
Similarly, if the time is included in the state, we would be able to model dynamic characteristics, e.g. that certain types of robots are only available during regular working hours, when humans are available to supervise.
Of course, in simple cases, $g$ might be state-independent.
For $s \in \mS$, $X$ is a sampling procedure parameterized by the tuple $(s, M, U, C, \bm{\phi}_M, g_U, g_C)$, where $\bm{\phi}_M \in [0, 1]^M$ represents the logits of a categorical distribution:
\begin{enumerate}
    \vspace{-5pt} \item Sample an integer $N$ from \textit{Cat}$(\bm{\phi}_M)$. 
    \vspace{-5pt} \item Sample $N$ controlled agents via $g_U$ from $U$. 
    \vspace{-5pt} \item Sample $M-N$ uncontrolled agents via $g_C$ from $C$.
    \vspace{-5pt}
\end{enumerate}

Note that the agent sampling functions $g_U, g_C$ could employ sampling with replacement to model scenarios where two robots may use the same policy, or sampling without replacement for scenarios where that is not possible (e.g. if all robots are heterogeneous). 

\subsection{Further Discussion of the Motivating Example}
\label{app:motivating_example}
This section provides proofs for claims in Section \ref{sec:motivation}, and further discussion in support of the motivating example.

\begin{lemma} \label{lemma:optimal-p-static}
    For a team of $M$ agents independently and identically selecting actions with probability $p$, the $p$ that maximizes the probability of winning is $p=\frac{1}{M}$. 
\end{lemma}

\begin{proof}
The probability of winning, $P(s=1)$ may be computed as follows: 

\begin{equation*}
    P(s=1) = M * P(\text{agent } i \text{ chooses 1 and $i^-$ choose 0}) 
        = M p (1-p)^{M-1}.
\end{equation*}

We wish to analytically compute the global maxima of this expression with respect to $p$, by considering the boundary points $p \in \{0, 1\}$ and the zeros of the first derivative. 

Clearly, if all agents select 1 with probability 0, the team cannot win; thus, $p=0$ is not a maxima. 
Similarly, if all agents select 1 with probability 1, the team also never wins; thus, $p=1$ is also not a maxima. 
Next, we compute the derivative of $P(s=1)$: 

\begin{align*}
    \frac{d[P(s=1)]}{dp} &= M[(1-p)^{M-1} - (M-1) p (1-p)^{M-2}] \\
    &= M[(1-p)^{M-2} [(1-p) - (M-1)p]].
\end{align*}

The zeros of the above expression occur at $p=0$ and $p=\frac{1}{M}$. For $p=\frac{1}{M}$, note that $P(s=1) = (\frac{M-1}{M})^{M-1} > 0$; thus, $p=\frac{1}{M}$ must be the global maximum.  
\end{proof}

\begin{lemma} \label{lemma:aht-scenario}
In a team of three agents consisting of two uncontrolled agents who select 1 with probability $p=\frac13$, and one controlled ad hoc agent who selects 1 with probability $p_{\aht}$, the probability of winning is $P(s=1) = \frac49$.    
\end{lemma}

\begin{proof}
Let $s_u$ denote the sum of the bits chosen by the two uncontrolled agents, and $b_{aht}$ denote the bit value chosen by the controlled agent. 
The probability of winning can be computed by partitioning the winning outcomes by whether $s_u = 1$ and the ad hoc agent selects 0, or whether the ad hoc agent selects 1 and both uncontrolled agents select 0, $s_u = 0$. 

\begin{align*}
    P(s=1) &= P(s_u = 1 \land b_{aht} = 0) + P(s_u = 0 \land b_{aht} = 1) \\ 
    &= 2 \cdot \frac13 \cdot \frac23 \cdot 
     \bigl( 1 - p_{\aht} \bigr) + p_{\aht}\left(\frac23 \cdot \frac23\right)\\
    &= \frac49 - \frac49 \cdot p_{\aht} + \frac49 \cdot p_{\aht} \\ 
    &= \frac49.
\end{align*}
\end{proof}

\begin{lemma} \label{lemma:aht-in-naht-setting}
In a team of three agents where two ad hoc agents select 1 with probability $p_{\aht}$ and the remaining uncontrolled agent selects 1 with probability $\frac13$, the maximizing $p_{\aht} = \frac13$ and the corresponding winning probability is $P(s=1) = \frac49$.
\end{lemma}

\begin{proof}
Let $s_{aht}$ denote the sum of the bits chosen by the two ad hoc agents, and $b_{u}$ denote the bit value chosen by the uncontrolled agent. 
The probability of winning may be directly computed as follows: 
\begin{align*}
    P(s=1) &= P(s_{aht} = 1 \land a_u = 0) + P(s_{aht} = 0 \land a_u = 1) \\
    &= 2 \cdot p_{\aht} \cdot (1 - p_{\aht}) \cdot \frac23 + (1 - p_{\aht})^2 \cdot \frac13 \\
    &= (1 - p_{\aht}) \left[ \frac43 p_{\aht} + \frac13 (1 - p_{\aht}) \right] \\ 
    &= (1 - p_{\aht}) \left(\frac13 + p_{\aht}\right).
\end{align*}
To determine the maximizing value, we compute $P(s=1)$ at the boundary points $p_{\aht} \in \{0, 1\}$ and at the zeros of the derivative of $P(s=1)$. 
Note that for $p_{\aht}=0$, $P(s=1) = \frac13$, while for $p_{\aht} = 1$, $P(s=1) = 0$. 

The derivative of the analytic expression of $P(s=1)$ with respect to $p_{\aht}$ is: 

\begin{equation}
    \frac{d [ P(s=1)]}{d p_{\aht}} = (1 - p_{\aht}) - \left(\frac13 + p_{\aht}\right) = \frac23 - 2p_{\aht}. 
\end{equation}

The zeros of the above expression occur at $p_{\aht} = \frac13$, which corresponds to $P(s=1) = \frac49$.
\end{proof}

\begin{lemma} \label{lemma:two-player-setting}
    In a team of three agents where one agent always selects 0, one agent selects 1 with probability $\frac13$, and the last agent selects 1 with probability $p_{\naht}$, the optimal $p_{\naht} = 1$ and results in a winning probability of $P(s=1) = \frac23$.
\end{lemma}

\begin{proof}
Since one of the two controlled agents always plays 0, the game effectively becomes a two-player game instead. 

Let $a_u$ denote the action selected by the uncontrolled agent, and let $a_{naht}$ denote the action of the controlled agent.
The probability of winning may be directly computed as follows: 
\begin{align*}
    P(s=1) &= P(a_u = 1 \land a_{naht} = 0) + P(a_u = 0 \land a_{naht} = 1) \\
    &= \frac13 \cdot (1 - p_{\naht}) + \frac23 \cdot p_{\naht} \\ 
    &= \frac 13 + \frac13 \cdot p_{\naht}.
\end{align*}

It is clear that $p_{\naht} = 1$ maximizes $P(s=1)$, and the corresponding value of $P(s=1) = \frac23$.
\end{proof}

\subsection{Experiment Details}
\label{app:exp_details}

\subsubsection{Team Sampling Procedure Used in Experiments}
\label{app:exp_team_sampling_proc}

\begin{wrapfigure}{r}{0.4\textwidth} %
    \centering
    \includegraphics[width=0.4\textwidth]{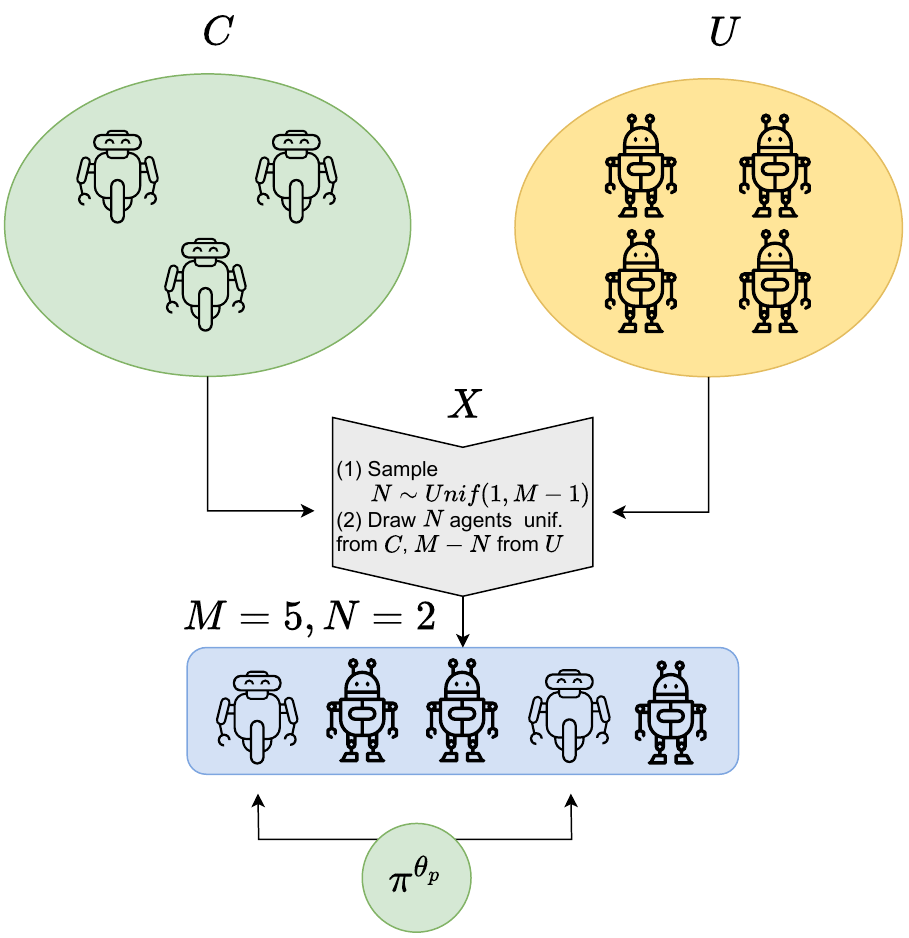}
    \caption{A practical instantiation of the NAHT problem.}
    \label{fig:naht_train}
\end{wrapfigure}

Recall that $U$ and $C$ denote the sets of uncontrolled and controlled agent policies, respectively, while $M$ denotes the team size for the task, and $N$ the number of agents sampled from $C$. 
The experiments in the following consider an $X_{train}$ consisting of sampling $N$ uniformly from $\{1, \cdots, M-1\}$, sampling $N$ agents from $C$ and $M-N$ agents from $U$ in a uniform fashion (Figure~\ref{fig:naht_train}). 
The sampling procedure takes place at the beginning of each episode to select a team, which is then deployed in the environment. Data generated by the deployed team (e.g. joint observations, joint actions, and rewards) are returned to the learning algorithm. See Appendix~\ref{app:eval_details} for further details on the evaluation procedure.

\subsubsection{Experimental Domains}
\label{app:exp_setting}

\paragraph{MPE Predator Prey}
The MPE environment \citep{lowe2017multi} (released under the MIT License) is a setting where particle agents interact within a bounded 2D plane, equipped with a discrete action space and continuous observation space. The observation space is 16-dimensional, and contains the agents' own position and velocity, relative positions/velocities of all other agents, landmarks, and prey. The action space consists of five discrete actions, corresponding to the four cardinal movement directions and a no-op action. The observation range is normalized to $[-1,1]$, while the discrete action space is one-hot encoded.

The predator-prey task (\texttt{mpe-pp}) is a custom task implemented by the authors of this paper within the fork of the MPE environment released by \citet{papoudakis_benchmarking_2021} (MIT License).
In this task, three predators must cooperatively pursue a pre-trained prey agent. 
We use the pre-trained prey policy provided by the ePymarl MARL framework. 
The prey policy was originally trained by \citet{papoudakis_benchmarking_2021}, by using the MADDPG MARL algorithm to train both predator and prey agents for 25M steps, and generally attempts to escape approaching predators.
The team receives a reward of 1 if two or more predators collide with the prey at a single time step, and no reward if only a single agent collides with the prey. 
A shaping reward consisting of 0.01 times the $\ell_2$ distance between each agents in the team and the prey is provided as well. 
Since the prey policy is pre-trained and fixed, note that our predator-prey task is a fully cooperative task from the perspective of the predator (learning) agents.
The maximum episode length is 100 time steps.

\paragraph{SMAC}
SMAC \citep{samvelyan19smac} (released under the MIT License) features a set of cooperative tasks, where a team of allied agents must defeat a team of enemy agents controlled by the game server. 
It is a partially observable domain with a continuous state space and discrete action space. 
For each agent, the observation space is continuous, and consists of features about itself, enemy, and allied agents within some radius. 
The action space is discrete, and allows an agent to choose an enemy to attack, a direction to move in, or to not perform any action.
As the number of allies and enemies varies between tasks, the dimensionality of the observation and action spaces is particular to each task. 
The observation space is normalized to be between $[-1, 1]$, while the action space is one-hot encoded.
The maximum number of time steps per episode is also task specific, although early termination occurs if all enemies are defeated.
At each time step, the team receives a shaped reward corresponding to the damage dealt, and bonuses of 10 and 200 points for killing an enemy and winning the scenario by defeating all enemies. 
The reward is scaled such that the maximal return achievable in each task is 20.

The SMAC tasks considered in this paper are described in more detail below: 
\begin{itemize}
    \item \texttt{5v6}: stands for 5m vs 6m, five allied Marines versus six enemy Marines.
    \item \texttt{8v9}: stands for 8m vs 9m, eight allied Marines versus nine enemy Marines.
    \item \texttt{10v11}: stands for 10m vs 11m, ten allied Marines versus eleven enemy Marines. 
    \item \texttt{3s5z}: stands for 3s vs 5z, three allied Stalkers versus five enemy Zealots. 
\end{itemize}

\subsubsection{Evaluation Details}
\label{app:eval_details}

This section provides details on how mixed teams are evaluated in the NAHT setting, and how cross-play scores, and self-play scores, and uncertainty measures are computed. 

\paragraph{$M-N$ score.} 
Given a set of controlled agents $C$, and a set of uncontrolled agents $U$, the goal of the $M-N$ score is to quantify the performance when these two teams must cooperate within the NAHT scenario. 
The $M-N$ score is computed in a deterministic and exhaustive fashion, by iterating over all possible values of $N$. 
Let $N$ be the number of agents sampled from set $C$, such that $N < M$. For $N \in \{1, \cdots, M-1\}$, construct the joint policy $\bm{\pi}^{(M)}$ by selecting $N$ agents uniformly from $C$ and $M-N$ agents from $U$. Evaluate the resultant team on the task for $E$ episodes. This results in $(M-1)*E$ episode returns, which is averaged to form the $M-N$ score. 
\paragraph{Cross-play scores.} The cross-play scores reported in this paper are the average returns of teams generated by algorithm $A$, when coordinating with those generated by $B$. To compute the cross-play score, we first train multiple teams (by varying the seed) via both algorithms $A$ and $B$. Next, the $M-N$ score is computed for random pairings of teams (seeds) from $A$ and $B$, following the procedure specified above. Summary statistics can be computed over the set of all such NAHT returns. 

For example: each algorithm (VDN, QMIX, IQL, MAPPO, IPPO) is run $k$ times with different seeds, to generate $k$ teams of agents that may act as uncontrolled teammates. For each pair of algorithms, we sample a subset of the possible seed \textit{pairs}, and evaluate the teams that result from merging said seed pairs.
If VDN and QMIX have seeds 1 2, and 3, the cross-play evaluation might consider the seed pairings (VDN 1, QMIX 2), (VDN 2, QMIX 3), (VDN 3, QMIX 1). Given a pair of seeds (e.g. (VDN 1), (QMIX 2)), the $M-N$ cross-play score is computed as the average return generated by sweeping $N \in \{1, \cdots, M-1\}$ and evaluating the merged team that consists of selecting $N$ agents from the first team, and $M - N$ agents for $E$ episodes. In our experiments, $E = 128$, and $k=5$. 

\paragraph{Self-play scores.} The self-play scores reported in this paper are \textit{model} self-play score, rather than \textit{algorithm} self-play score \citep{marl-book-albrecht24}. The reason for this is that we are interested in the performance of agents when paired with \textit{known teammates}. 

\paragraph{Measuring uncertainty} Means and 95\% confidence intervals are computed over all team/seed pairings considered for any given scores, as we treat each $M-N$ evaluation as an independent trial. For most experiments, five seed pairs are considered, where seeds are paired to ensure that all seeds participate in at least one evaluation. This ensures  that the computational cost of the evaluation remains linear in $N$, rather than quadratic. However, note that for the OOD experiments presented in Section \ref{sec:ood_generalization}, we consider all possible seed pairings for the most comprehensive evaluation.

\subsubsection{Algorithm Implementation}
\label{app:alg}

The experiments in this paper use algorithm implementations from the ePyMARL codebase \citep{papoudakis_benchmarking_2021} (released under the Apache License). The value-based methods are used without modification (i.e. IQL, VDN, QMIX), but we implement our version of policy gradient methods (IPPO, MAPPO), based on the implementation of \citet{yu2021surprising}. 

\paragraph{Parameter Sharing.} All methods employ recurrent actors and critics, with full parameter sharing, i.e. all agents are controlled by the same policy, where the agent id is input to the actor and critic networks, to allow behavioral differentiation between agents. 
For POAM, which also maintains encoder-decoder networks for agent modeling, parameter sharing is used for the encoder and decoder networks, to improve training sample efficiency. Further, POAM's decoder, which predicts observations and actions for all $-i$ agents, employs parameter sharing across agent predictions, to prevent the target dimensionality from scaling with the number of teammates.

\paragraph{Optimizer and Neural Architecture.} The Adam optimizer is applied for all networks involved.  
For policy gradient methods, the policy architecture is two
fully connected layers, followed by an RNN (GRU) layer, followed by an output layer.
Each layer has 64 neurons with ReLU activation units, and employs layer normalization.
The critic architecture is the same as the policy architecture.
The value-based methods employ the same architecture, except that there is only a single fully connected layer before the RNN layers, and layer normalization is not used, following the ePyMARL implementation. 
Please consult the codebase for full implementation details.

\begin{table}[h]
    \centering
    \begin{tabular}{|c|c|c|c|p{1.5cm}|c|p{1.5cm}|} \hline 
         Algorithm&  Buffer size&  Epochs&  Minibatches&  Entropy&  Clip&  Clip value loss\\ \hline 
         IPPO&  128, \textbf{256}, 512&  1, \textbf{4}, 10&  1, \textbf{3}&  0.01, 0.03, \textbf{0.05}&  0.01, 0.05, \textbf{0.1}&  \textbf{no}, yes\\ \hline 
         MAPPO&  64, 128, \textbf{256}&  4, \textbf{10}&  \textbf{1}, 3&  0.01, \textbf{0.03}, 0.05, 0.07&  \textbf{0.05}, 0.1, 0.2&  no, \textbf{yes}\\ \hline
    \end{tabular}
    \vspace{2pt}
    \caption{Hyperparameters evaluated for the policy gradient algorithms. Selected values are bolded}
    \label{tab:hyperparam_pg}
\end{table}

\paragraph{Hyperparameters.} For the value-based methods, default hyperparameters are used. We tune the hyperparameters of the policy gradient methods on the \texttt{5v6} task, and apply those parameters directly to the remaining SMAC tasks. The hyperparameters considered for policy gradient algorithms are given in Table \ref{tab:hyperparam_pg}. POAM adopts the same hyperparameters as IPPO where applicable. We also tuned additional hyperparameters specific to POAM (Table~\ref{tab:poam_hyper}).

\begin{table}[h]
    \centering
    \begin{tabular}{|c|c|c|c|l|} \hline 
         Task & Algorithm&  ED epochs& ED Minibatches &ED LR\\ \hline 
         \texttt{5v6} & POAM&  \textbf{1}, 5, 10&  \textbf{1}, 2&\textbf{0.0005}, 0.005\\ \hline
         \texttt{mpe-pp} & POAM&  \textbf{1}, 5&  \textbf{1}, 2&\textbf{0.0005}, 0.001\\ 
         \hline
    \end{tabular}
    \vspace{10pt}
    \caption{Additional hyperparameters evaluated for POAM; note that ED stands for encoder-decoder. Selected values are bolded.}
    \label{tab:poam_hyper}
\end{table}

\subsection{Supplemental Figures}
\label{app:supp_results}

This section contains additional figures referenced by the main paper. Please see Section \ref{sec:exp} for the corresponding analysis and discussion.

\begin{figure}[h!]
\centering
\begin{minipage}{.55\textwidth}
  \centering
    \includegraphics[width=0.85\textwidth]{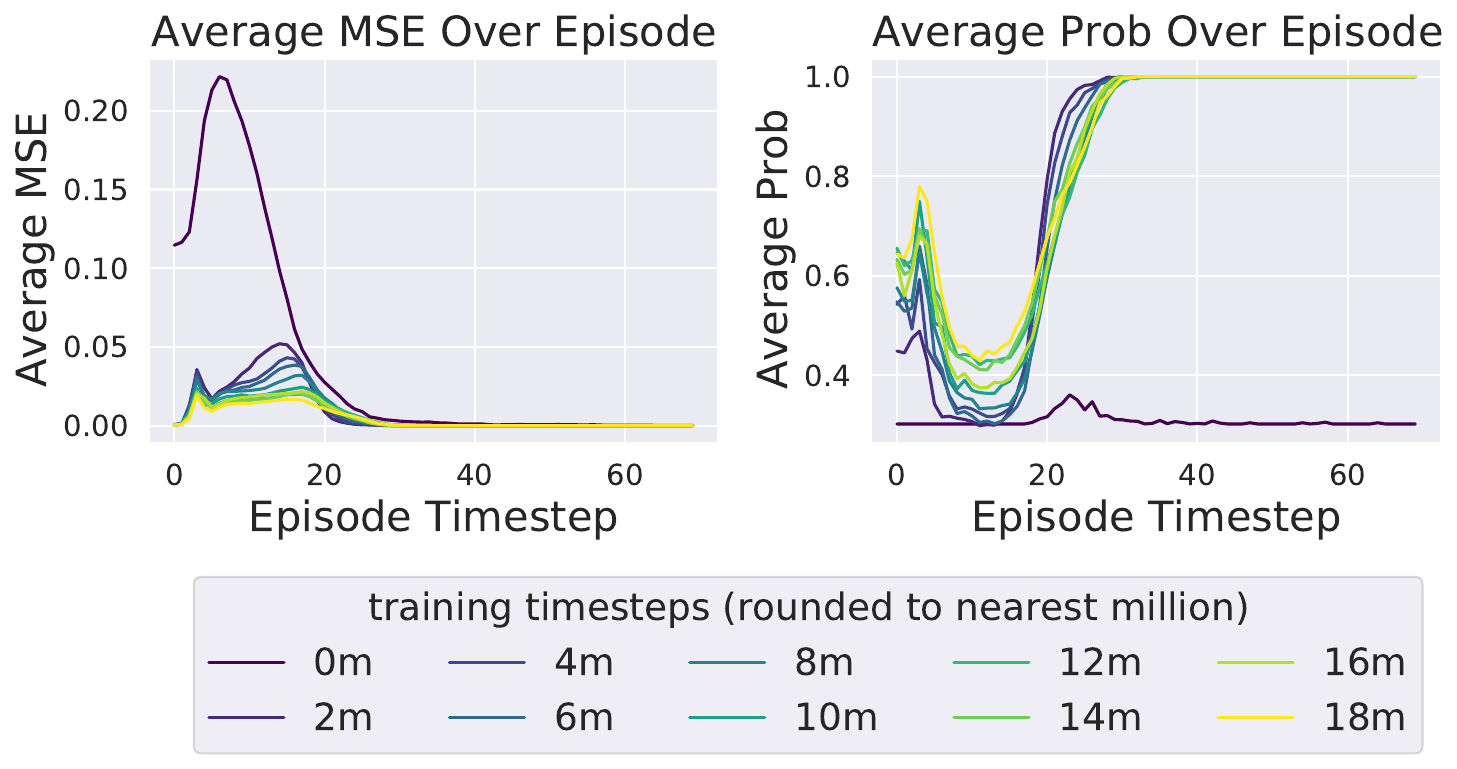}
    \caption{Evolution of the within-episode mean squared error (left) and probability of actions that were actually taken by modeled teammates (right), computed by the POAM agent over the course of training, on the \texttt{5v6} task. 
    The agent modeling performance improves over the course of an episode, as more data about teammate behavior is observed.
    }
    \label{fig:5v6_poam_within_episode}
\end{minipage}%
\hspace{5pt}
\begin{minipage}{.35\textwidth}
    \centering
    \vspace{-5pt}
    \includegraphics[width=1.0\textwidth]{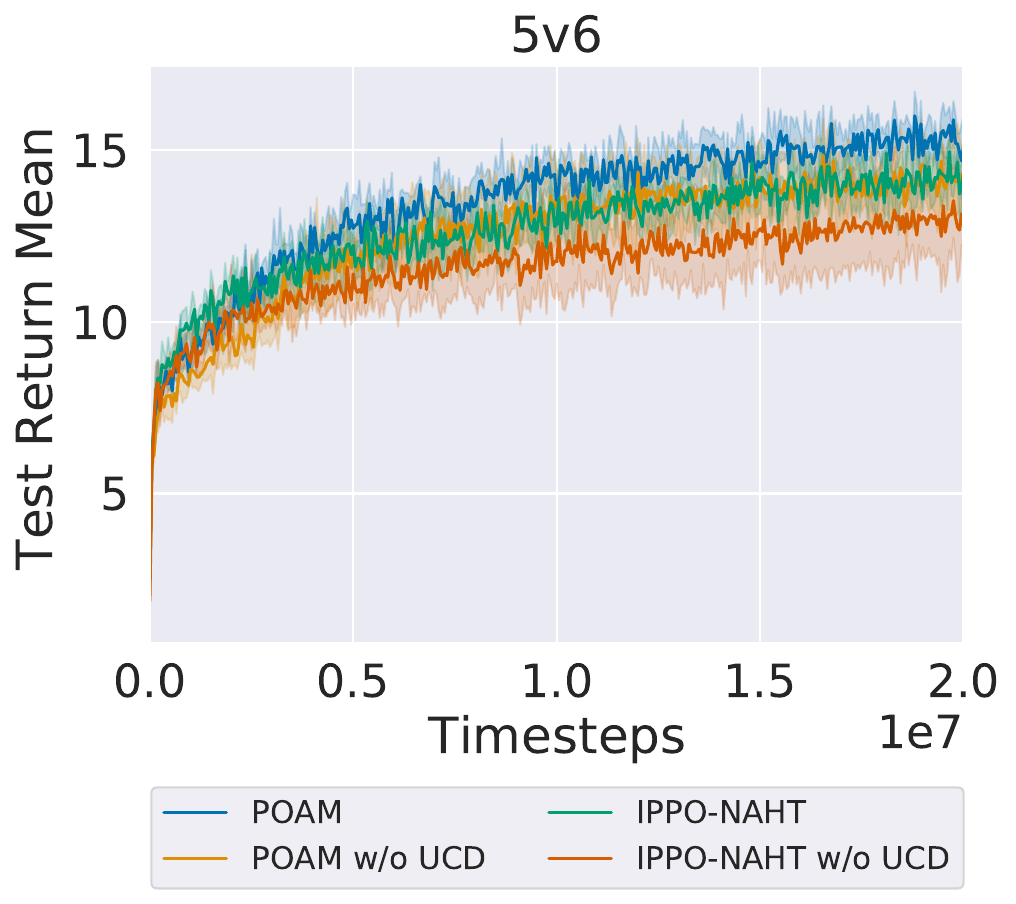}
    \caption{Learning curves of POAM and IPPO, where the value network is trained with and without uncontrolled agents' data (UCD) on \texttt{5v6}. }
    \label{fig:5v6_critic_masking}
\end{minipage}
\end{figure}

\begin{figure}[h!]
    \centering
    \includegraphics[width=1.0\textwidth]{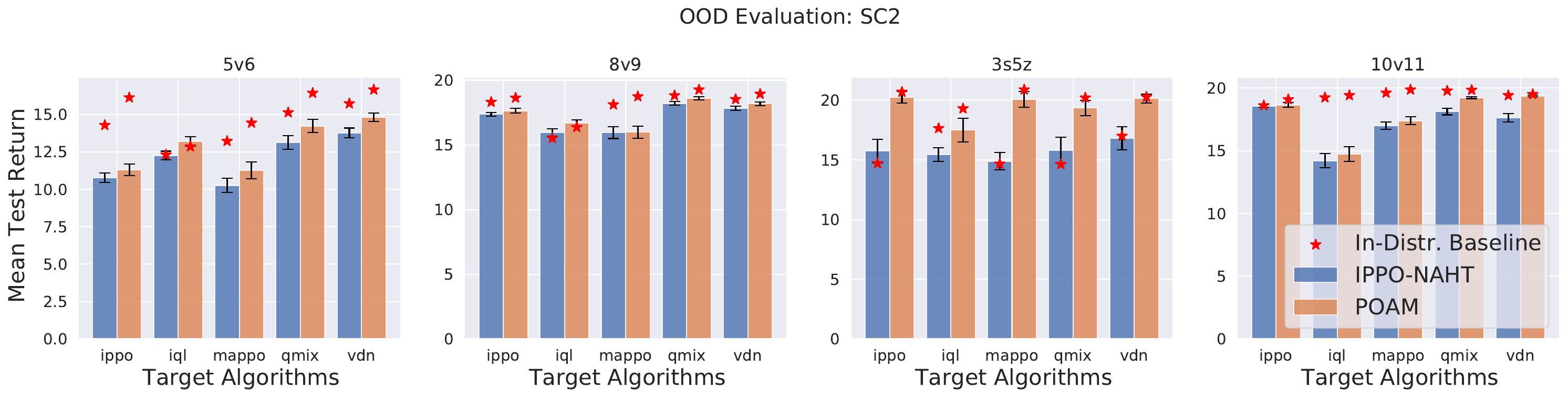}
    \caption{Returns achieved by POAM and IPPO-NAHT, when paired with out-of-distribution teammates. POAM shows improved generalization to OOD teammates, as compared to IPPO-NAHT, across all StarCraft tasks. 
    For each type of teammate, the performance of IPPO-NAHT/POAM against the exact teammate seen during training is shown as the in-distribution baseline.}
    \label{fig:sc2_ood_gen}
\end{figure}

\subsection{Supplemental Analysis}

The following subsections contains secondary analysis and results, intended to support the primary analysis of the main paper. 
For all results, the mean and 95\% confidence interval over five trials is reported. All reported returns are \textit{test} returns.

\subsubsection{The Need for Dedicated NAHT Algorithms - Empirical Evidence}
\label{app:supp_results:need_for_naht_exp}

Section \ref{sec:motivation} argues theoretically that with a population of uncontrolled teammates that select one with probability 1/3 and zero otherwise, an optimal policy learned in the AHT setting would not be optimal in the NAHT setting. Here, we present experiments on the three agent bit matrix game that empirically verify the theory. Let $N$ denote the number of controlled agents. The episode length is 25, and the observation for each agent consists of the agent index and the joint action at the previous time step. The team reward at each time step is $3*\mathbbm{1}_{\sum_i b_i = 1}$.

The optimal expected return in the $N=1$ (AHT) setting is 33.333 (derived from Lemma A.3), and in the $N=2$ (NAHT) setting is  50.0 (derived from Lemma A.4).  The methods compared are POAM-AHT and POAM. Table \ref{tab:need_for_naht} shows the expected returns achieved by POAM and POAM-AHT on the $N=1$ and $N=2$ scenarios. Since any policy is optimal in the $N=1$ case, as expected, both methods achieve near-optimal returns. In the $N=2$ case, as expected, POAM outperforms POAM-AHT by a large margin, achieving a near-optimal return of $48.858 \pm 1.092$.
 
\begin{table}[hb]
    \centering
    \begin{tabular}{|c|c|c|} \hline 
         &  N=1 (AHT)& N=2 (NAHT)\\ \hline 
         POAM-AHT&  33.441 $\pm$ 0.511& 22.5 $\pm$ 8.250\\ \hline 
         POAM&  33.483 $\pm$ 0.485& 48.858 $\pm$ 1.092\\ \hline
    \end{tabular}
    \vspace{10pt}
    \caption{Returns of POAM-AHT versus POAM on the three agent bit matrix game, in the N=1 (AHT) and N=2 (NAHT) setting. POAM and POAM-AHT both achieve the optimal return for the $N=1$ case, but POAM has a much higher return on the $N=2$ case.}
    \label{tab:need_for_naht}
\end{table}

\subsubsection{Validating the Existence of Coordination Conventions}
\label{app:supp:coord_conventions}

The experimental procedure detailed in Section~\ref{sec:ExpDesign} generates diverse teammates in SMAC and MPE by using  MARL algorithms to train multiple teams of agents. Two underlying assumptions of the procedure are that (1) teams trained by the same algorithm learn non-compatible coordination conventions, and (2) teams trained by different algorithms are not compatible. Both points are experimentally verified in this section.

\paragraph{Self-play with MARL algorithms.}
For the tasks under consideration, teams trained using the same MARL algorithm, but different seeds, can converge to distinct coordination conventions. Figure \ref{fig:sp_scores} demonstrates this by depicting the return of teams trained together (matched seeds) versus those of teams that were not trained together (mismatched seeds), across all naive MARL algorithms (IPPO, IQL, MAPPO, QMIX, VDN) and tasks considered. 
Overall, the returns of the teams that were trained together are higher than those not trained together.

The general phenomenon has been previously observed and exploited by prior works in ad hoc teamwork \citep{strouse_fcp_2022}. This paper takes advantage of this to generate a set of diverse teammates for the StarCraft II experiments. We select tasks where the effect is significant, to ensure that there are distinct coordination behaviors for POAM to model. Tasks that were considered but subsequently ruled out for this reason include \texttt{3m vs 3m, 8m vs 8m}, \texttt{6h vs 8z}, and the MPE Spread task.

\begin{figure}[t]
    \centering
    \includegraphics[width=1.0\textwidth]{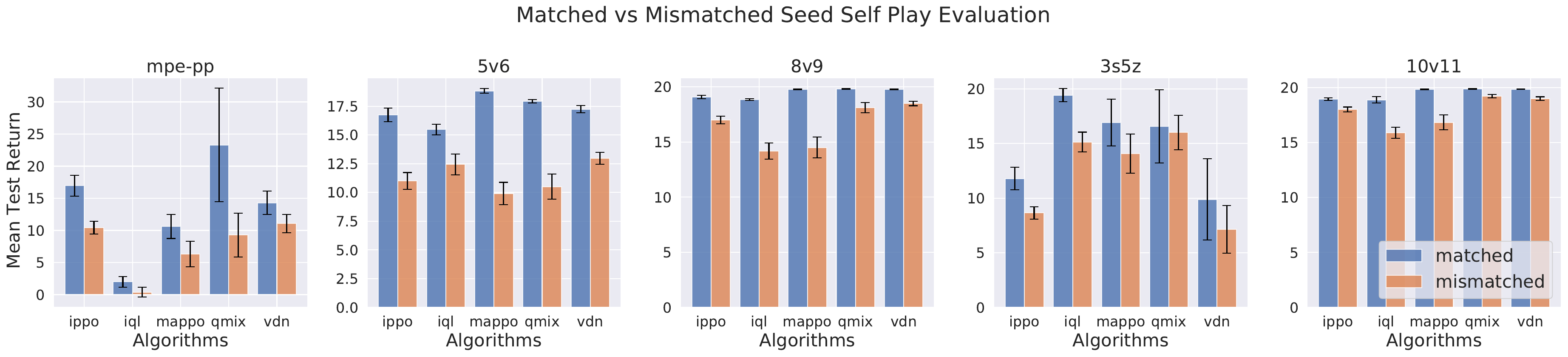}
    \caption{Agent teams that were trained together using the same algorithm (\texttt{matched} seeds) have higher returns than teams that were not trained together but were trained with the same algorithm (\texttt{mismatched} seeds).}
    \label{fig:sp_scores}
\end{figure}

\paragraph{Cross-play with MARL algorithms}

\begin{table}[h!]
    \centering
    \includegraphics[width=\textwidth]{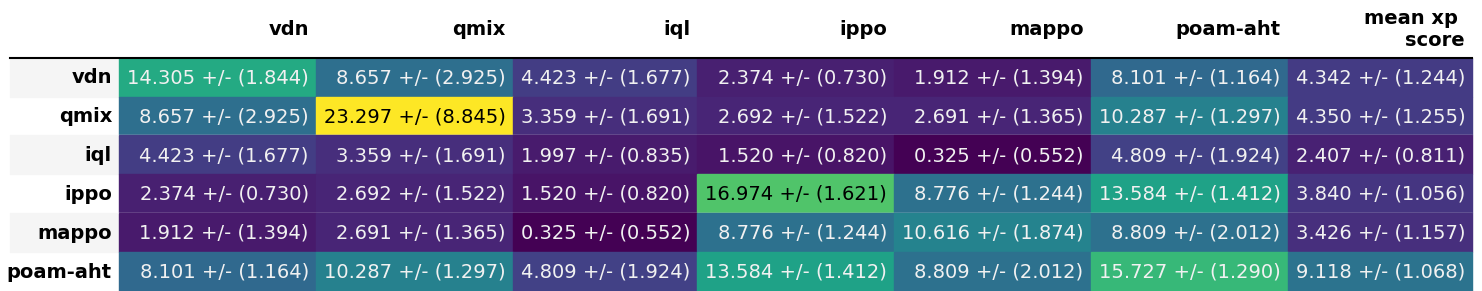}
    \caption{Cross-play results for the \texttt{mpe-pp} task.}
    \label{fig:mpe-pp-xp-matrix}
\end{table}

\begin{table}[h!]
    \centering
    \includegraphics[width=\textwidth]{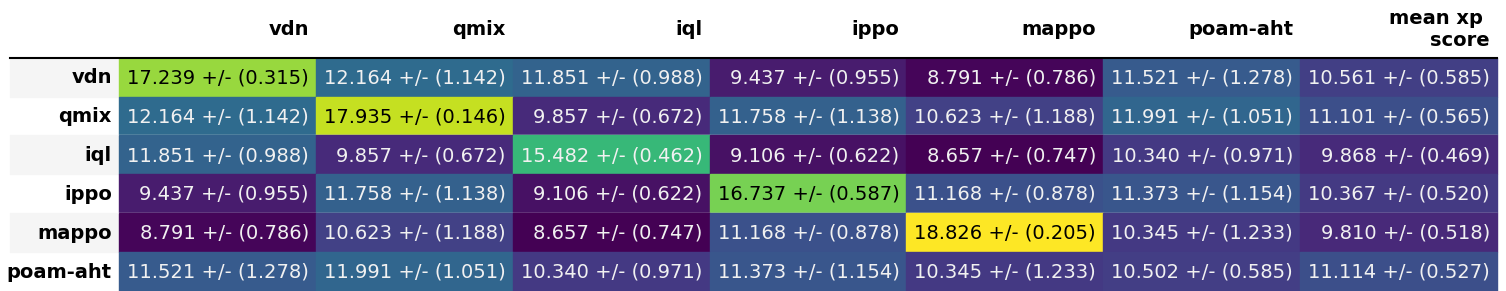}
    \caption{Cross-play results for the \texttt{5v6} task.}
    \label{fig:5v6-xp-matrix}
\end{table}

\begin{table}[h!]
    \centering
    \includegraphics[width=\textwidth]{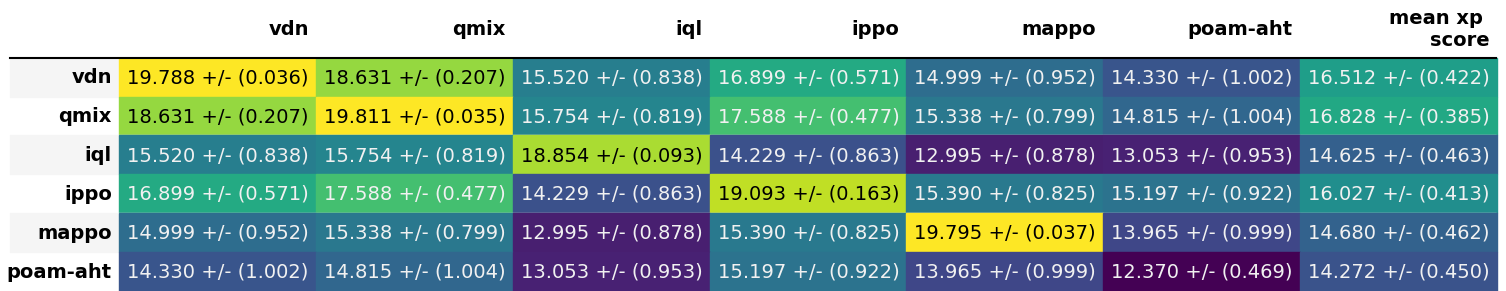}
    \caption{Cross-play results for the \texttt{8m} task.}
    \label{fig:8v9-xp-matrix}
\end{table}

\begin{table}[h!]
    \centering
    \includegraphics[width=\textwidth]{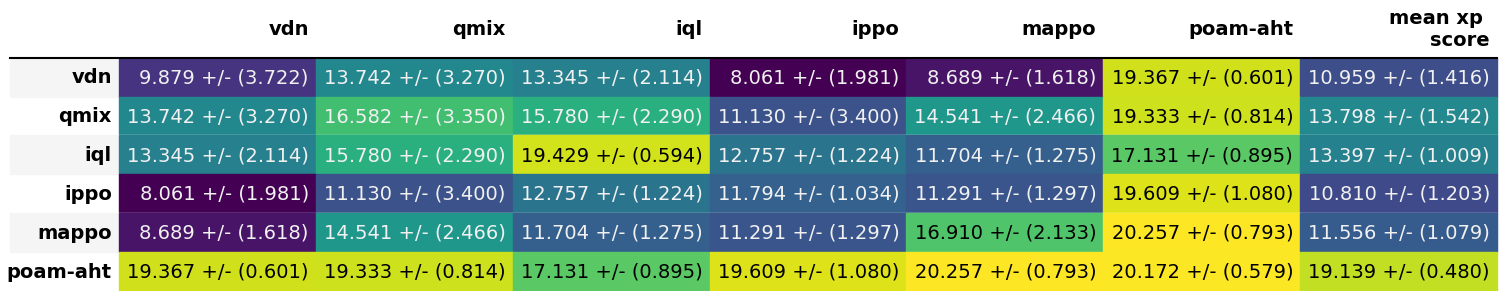}
    \caption{Cross-play results for the \texttt{3s5z} task.}
    \label{fig:3s5z-xp-matrix}
\end{table}

\begin{table}[h!]
    \centering
    \includegraphics[width=\textwidth]{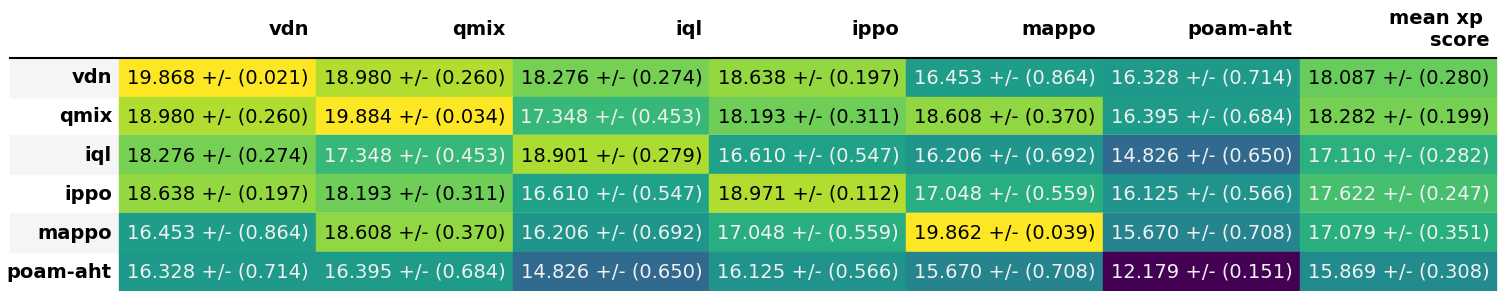}
    \caption{Cross-play results for the \texttt{10v11} task.}
    \label{fig:10v11-xp-matrix}
\end{table}

Tables \ref{fig:mpe-pp-xp-matrix}, \ref{fig:5v6-xp-matrix}, \ref{fig:8v9-xp-matrix}, \ref{fig:3s5z-xp-matrix}, and \ref{fig:10v11-xp-matrix} display the full cross-play results for all MARL algorithms and POAM-AHT, on all tasks. Note that the tables are reflected across the diagonal axis for viewing ease. The values on the off-diagonal (i.e., where the two algorithms are not the same) are the cross-play score, while the values shown on the diagonal are self-play scores, computed as described in App. \ref{app:eval_details}. 
The cross-play and self-play scores displayed are means and  95\% confidence intervals. 
The rightmost column reflects the row average and 95\% confidence interval of the cross-play (XP) scores corresponding to the test set of VDN, QMIX, IQL, IPPO, and MAPPO, excluding the self-play score. 
Thus, the rightmost column reflects how well on average the row MARL/AHT algorithm can generalize to the test set used throughout this paper.   

Overall, we find that MARL algorithms perform significantly better in self-play than cross-play. We note that there are a few exceptions (e.g., IPPO vs QMIX on \texttt{8m}), and that the cross-play and self-play scores are much closer on \texttt{10v11}, but overall the trend is consistent across tasks. 

\subsubsection{Generalization to Unseen Teammate Types}
\label{app:supp:alt_ood_gen}

As discussed and experimentally verified in Section \ref{app:supp:coord_conventions}, diverse coordination conventions in StarCraft and MPE Predator Prey can be generated by (1) running the same MARL algorithm with various random seeds, or (2) running various MARL algorithms. 
In Section \ref{sec:ood_generalization}, the out-of-distribution (OOD) teammates considered were generated by running MARL algorithms with different random seeds than those used to generate train-time teammates---in other words, generating diverse and unseen teammates using the first 
procedure specified above.

\begin{figure}[h!]
    \centering
    \includegraphics[width=1.0\textwidth]{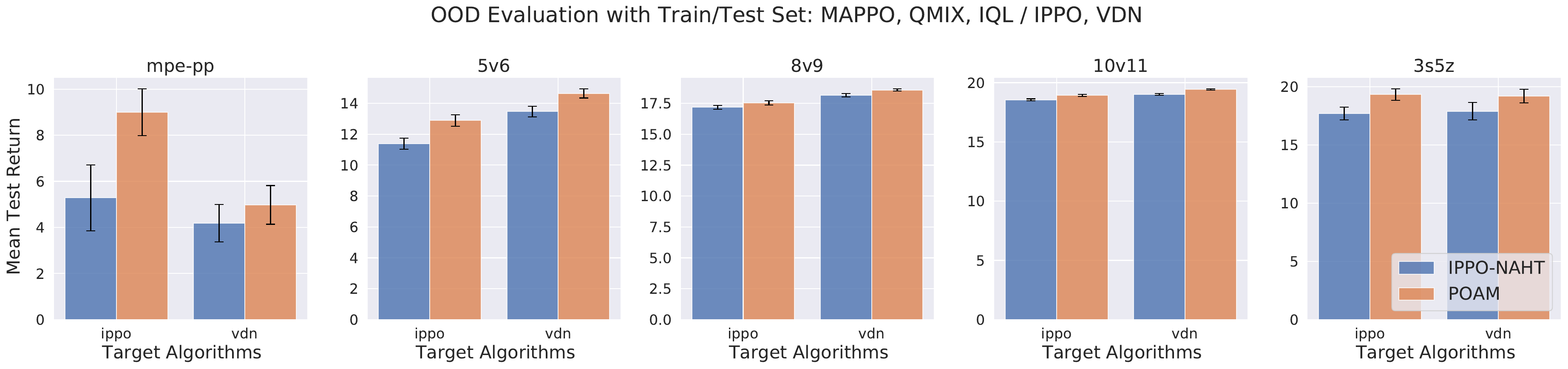}
    \caption{Returns achieved by POAM and IPPO-NAHT, when trained on MAPPO, QMIX, and IQL, and tested on IPPO, VDN.}
    \label{fig:alt_ood_gen}
\end{figure}

Here, we generate OOD teammates using the second procedure. More precisely, the five MARL algorithms used to generate teammates may be divided into train/test sets. Figure \ref{fig:alt_ood_gen} shows the results of such an experiment, where POAM and IPPO-NAHT are trained with MAPPO, QMIX, and IQL teammates, and tested with agents from IPPO and VDN. Similar to the results presented in Section \ref{sec:ood_generalization}, we find that POAM generally outperforms IPPO-NAHT for the unseen teammate types for all tasks.

\subsubsection{Modeling Controlled and Uncontrolled Agents}
\label{app:supp:ed_ctrl_unctrl}

\begin{figure}[h!]
    \centering
    \includegraphics[width=0.60\textwidth]{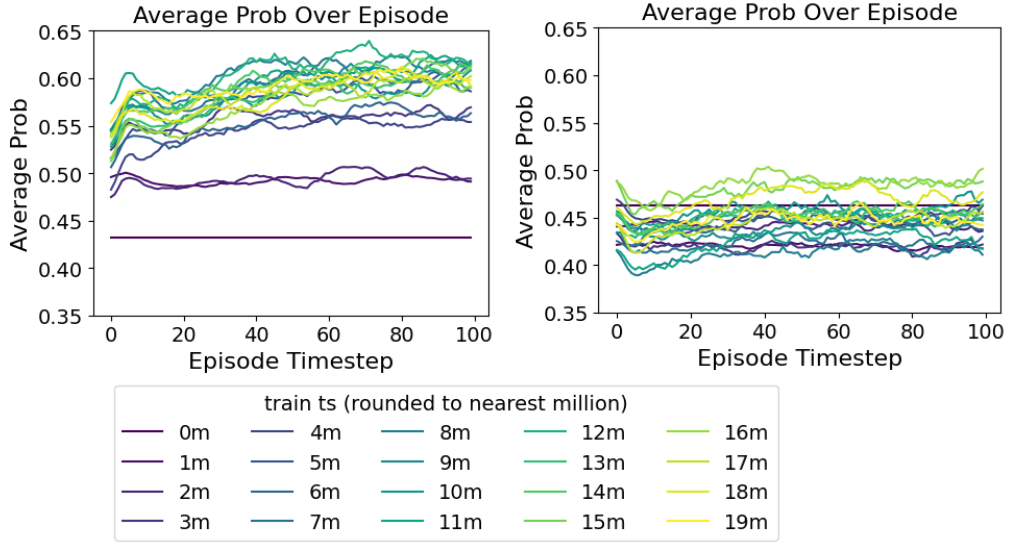}
    \caption{Loss of POAM's encoder-decoder on the \texttt{mpe-pp} task, separated by uncontrolled (left) and controlled  agents (right).}
    \label{fig:ed_loss_ctrl_unctrl}
\end{figure}

POAM’s encoder-decoder (ED) models both controlled and uncontrolled agents. 
Since the controlled agents are updated during the training process, the encoder-decoder must deal with a ``moving target". Thus, in principle, the problem of modeling the controlled agents is more challenging than modeling uncontrolled agents.

Fig.~\ref{fig:ed_loss_ctrl_unctrl} shows the probability of predicting the correct \textit{action} for the uncontrolled agents and controlled agents separately, on the \texttt{mpe-pp} task. 
Note that the action probabilities shown in Fig.~\ref{fig:mpe-pp_poam_within_episode} (right) of the main paper would be the average of the two plots shown in Fig. \ref{fig:ed_loss_ctrl_unctrl}.

For uncontrolled agents (Fig. \ref{fig:ed_loss_ctrl_unctrl}, left), we observe that the accuracy of action predictions for the uncontrolled agents increases much more consistently as training goes on, and is higher than that for the controlled agents. As expected, the ED is able to model the uncontrolled agents more easily than the nonstationary controlled agents.

\subsubsection{Performance as a Function of the Number of Controlled Agents}

To provide more insight on how an NAHT team's performance changes as the number of controlled agents increases, Figure \ref{fig:varying-n-ctrl} displays the mean test returns of POAM and POAM-AHT as a function of the number of controlled agents, where the evaluation returns are averaged across five types of uncontrolled teammates: QMIX, VDN, IQL, MAPPO, and IPPO. 
Note that the self-play returns of POAM and POAM-AHT (i.e. where the number of controlled agents is maximal) are shown as horizontal lines, while the performances at $N=0$ correspond to the averaged self-play returns of the five uncontrolled team types.

Recall that POAM-AHT was trained on the $N=1$ scenario only, while POAM is trained on the full NAHT setting. As expected, POAM outperforms POAM-AHT for all values of $N>1$, while for $N=1$, the methods perform similarly. POAM-AHT's performance declines as the number of controlled agents increases, which likely occurs because the evaluation setting becomes further from the training setting as $N$ increases.

\begin{figure}[h!]
    \centering
    \includegraphics[width=1.0\textwidth]{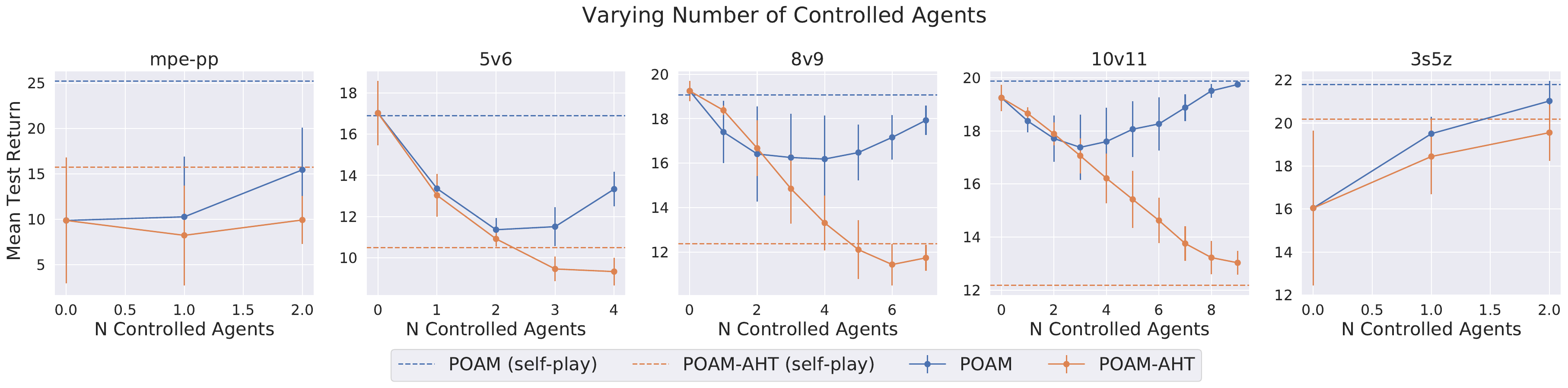}
    \caption{Comparing the performance of POAM versus POAM-AHT, as the number of controlled agents varies. POAM and POAM-AHT agents are evaluated with the following set of uncontrolled agents:  QMIX, VDN, IQL, MAPPO, IPPO.}
    \label{fig:varying-n-ctrl}
\end{figure}

\subsection{Computing Infrastructure}
\label{app:computing_infra}
All value-based algorithms (e.g. QMIX, VDN, IQL) were run without parallelizing environments, while policy gradient algorithms were run with parallelized environments. 
All methods were trained for 20M steps on all tasks. Each run took between 12-48 hours of compute, and used less than 2gB of GPU memory. Runs were parallelized to use computational resources efficiently.
The servers used for our experiments ran Ubuntu 20.04 with the following configurations:

\begin{itemize}
    \item Intel Xeon CPU E5-2630 v4;  Nvidia Titan V GPU. %
    \item Intel Xeon CPU E5-2698 v4; Nvidia Tesla V100-SXM2 GPU. %
    \item Intel Xeon Gold 6342 CPU; Nvidia A40 GPU. %
    \item Intel Xeon Gold 6342 CPU; Nvidia A100 Gpu. %
\end{itemize}